\documentclass[]{article}
\usepackage[margin=1.1in]{geometry}
\usepackage{amsmath}
\usepackage{amssymb}
\usepackage{amsthm}
\usepackage{mathrsfs}
\usepackage{mathtools}
\usepackage{bbm}
\usepackage{graphicx}
\usepackage{subcaption}
\usepackage{float}
\usepackage{xcolor}
\usepackage{tabularx}
\usepackage{booktabs}
\usepackage{todonotes}
\usepackage{footnote}  %
\usepackage{tikz}
\usetikzlibrary{decorations.pathmorphing}
\usepackage{microtype}
\usepackage{longtable}
\usepackage[export]{adjustbox}
\usepackage[ruled,vlined]{algorithm2e}
\usepackage{multicol}
\definecolor{linkcolor}{HTML}{449B52}
\definecolor{citecolor}{HTML}{67001F}
\definecolor{urlcolor}{HTML}{008080}
\usepackage[breaklinks,colorlinks=true,
            citecolor=citecolor,urlcolor=urlcolor,linkcolor=linkcolor]{hyperref}
\usepackage[noabbrev,capitalize]{cleveref}
\usepackage[
    backend=biber,
    uniquename=false,
    bibencoding=utf8,
    sorting=nty,
    doi=false,isbn=false,url=false,
    citestyle=ext-authoryear,
    maxcitenames=2,
    uniquelist=false,
    maxbibnames=123,
    backref=true,
    hyperref=true
]{biblatex}
\DeclareFieldFormat[article,inproceedings,manual,book,misc,thesis,report,dataset,incollection]{title}{\iffieldundef{url}{#1}{\href{\thefield{url}}{#1}}}
\DeclareInnerCiteDelims{cite}{(}{)}
\newcommand{\condensite}{\textit{condensité}\xspace}
\newcommand{\condensiteNN}{\textit{condensité~(NN)}\xspace}
\newcommand{\condensiteCNN}{\textit{condensité~(CNN)}\xspace}
\newcommand{\condensiteTree}{\textit{condensité~(tree)}\xspace}
\newcommand{\flexcode}{\textit{FlexCode}\xspace}
\newcommand{\flexcodeTree}{\textit{FlexCode~(tree)}\xspace}
\newcommand{\flexcodeKNN}{\textit{FlexCode~(knn)}\xspace}
\newcommand{\flexcodeCNN}{\textit{FlexCode~(CNN)}\xspace}
\newcommand{\drf}{\textit{DRF}\xspace}
\newcommand{\lincde}{\textit{LinCDE}\xspace}
\newcommand{\condensier}{\textit{condensier}\xspace}
\newcommand{\dd}{\mathrm{d}}
\newcommand{\F}{\mathcal{F}_{h}}
\newcommand{\G}{\mathcal{G}}
\newcommand{\N}{\mathcal{N}}

\newcommand{\E}{\mathbb{E}}

\newcommand{\Unif}{\operatorname{Unif}}
\newcommand{\Rad}{\operatorname{Rad}}
\newcommand{\argmin}{\mathop{\arg\min}}
\DeclareMathOperator{\Supp}{Supp}
\newcommand{\X}{\mathcal{X}}
\newcommand{\fstar}{f^\star}
\newcommand{\transformation}{f^\star_h}
\newcommand{\projection}{f^\star_{\F}}
\newcommand{\estimator}{\widehat{f}}

\renewcommand{\P}{P^\star}
\newcommand{\PX}{P_X^\star}
\newcommand{\Q}{Q^\star}
\newcommand{\QnM}{Q_{Mn}}
\newcommand{\Ph}{{\bar{P}_h^\star}}
\newcommand{\Phn}{\bar{P}_{h,n}}
\newcommand{\Radsum}{R_{h,n}}
\newcommand{\bfeps}{\boldsymbol{\epsilon}}

\newcommand{\ER}{\mathcal{E}}
\newcommand{\ERh}{\mathcal{E}_h}

\renewcommand{\L}{\mathcal{L}}

\newcommand{\NFQn}{\|F_h\|_{L^2(\QnM)}}

\newcommand{\phideltahat}{\phi_{h,n}(\hat\delta)}
\newcommand{\phiflat}{\phi^\flat_{h,n}}
\newcommand{\phisharp}{\phi^\sharp_{h,n}}
\newtheorem{theorem}{Theorem}
\newtheorem{proposition}{Proposition}
\newtheorem{assumption}{Assumption}
\newtheorem{lemma}{Lemma}

\title{Transforming Conditional Density Estimation Into a Single Nonparametric Regression Task\\[-1em]}
\author{}
\date{}

\begin{document}

\maketitle
\begin{center}
\begin{minipage}{0.9\linewidth}
\begin{center}
\begin{multicols}{2}
    \textbf{Alexander G. Reisach}\footnotemark\\
    CNRS, MAP5\\
    Université Paris Cité\\
    F-75006 Paris, France
    
    \columnbreak

    \textbf{Olivier Collier}\\
    MODAL’X, UPL\\
    Université Paris Nanterre\\
    F-92000 Nanterre, France

\end{multicols}
\begin{multicols}{2}

    \textbf{Alex Luedtke}\\
    Department of Health Care Policy\\
    Harvard University\\
    Boston, MA 02115, USA

    \columnbreak

    \textbf{Antoine Chambaz}\\   
    CNRS, MAP5\\    
    Université Paris Cité\\ 
    F-75006 Paris, France
\end{multicols}
\end{center}
\end{minipage}
\end{center}
\footnotetext{Correspondence to \url{alexander.reisach@math.cnrs.fr}}

\begin{abstract}
We propose a way of transforming the problem of conditional density estimation into a single nonparametric regression task via the introduction of auxiliary samples.
This allows leveraging regression methods that work well in high dimensions, such as neural networks and decision trees.
Our main theoretical result characterizes and establishes the convergence of our estimator to the true conditional density in the data limit.
We develop \href{https://github.com/Scriddie/condensite}{\condensite}, a method that implements this approach. 
We demonstrate the benefit of the auxiliary samples on synthetic data and showcase that \condensite can achieve good out-of-the-box results.
We evaluate our method on a large population survey dataset and on a satellite imaging dataset. In both cases, we find that \condensite matches or outperforms the state of the art and yields conditional densities in line with established findings in the literature on each dataset.
Our contribution opens up new possibilities for regression-based conditional density estimation and the empirical results indicate strong promise for applied research.
\end{abstract}
\section{Introduction}
Conditional density estimation is concerned with estimating the density of a random variable given a set of other covariate random variables.
It is one of the fundamental problems of statistics and of interest in applications relying on predictive models where summary statistics like the conditional mean or quantiles are insufficient.

Applications of conditional density estimation include many scenarios in economics, causal inference, and machine learning.
In the study of economic inequality, analyses often aim to decompose distributional differences in order to explain them. They can benefit from access to the conditional density for specific covariates to do so on a fine-grained level \parencite{Kneib2023Rage}. In finance, conditional densities are useful for credit risk assessment, asset return estimation, and options pricing.
In causal inference, the predominant statistical approaches \parencite{pearl2009causality,imbens2015causal} follow an interventionist account of causation \parencite[see][]{woodward2005making}, in which causal effects manifest as changes in distributions.
Hence, that describing conditional distributions is of central interest in causality and motivates several works on conditional density estimation, notably \cite{munoz2011super} and \cite{cevid2020distributional}. 
In machine learning, conditional density estimation can aid a range of methodologies, including anomaly detection, probabilistic forecasting, and risk-sensitive policy learning \parencite[see][respectively]{nachman2020anomaly,gneiting2014probabilistic,dabney2018distributional}.

On a rudimentary level, conditional density estimation targets the conditional density $y\mapsto f_{Y|X}(y|x)=f_{X,Y}(x,y)/f_X(x)$ of target $Y$ and covariates $X$ with joint density $(x,y)\mapsto f_{X,Y}(x,y)$ and marginal density $x\mapsto f_X(x)$ of the covariates.
We assume $Y$ to be univariate, and $X$ to be multivariate. 
The challenge is to estimate $f_{Y|X}$ from independently and identically distributed observations $(X_i, Y_i)_{1\leq i\leq n}$ following $f_{X,Y}$.
General ideas for approaching this problem center around either identifying values of $X$ with similar conditional densities and then performing estimation of the corresponding marginal, or on estimating $f_{X,Y}$ and $f_X$ separately and then combining them.
Both approaches face difficulties if $X$ is high-dimensional and data points therefore sparse, a phenomenon known as the ``curse of dimensionality'' \parencite{Bellman1961adaptive}. 
The curse of dimensionality makes density estimation a difficult problem in high dimensions \parencite[see e.g.][Section 4.5]{wasserman2006all}.
Conditional density estimation is usually thought to inherit this challenge, which would render it infeasible in many settings of interest, without prior dimensionality reduction. 
This contrasts with the spectacular success of high-dimensional regression for estimating summary statistics, in particular the mean, using methods based on neural networks and decision trees.
Transforming conditional density estimation into a regression problem that allows for leveraging such methods would open the possibility of transferring their success in high dimensions.
This goal motivates our contribution.

\paragraph{Contribution.}
We propose a way of estimating the conditional density directly and simultaneously for all data points via a single nonparametric regression using auxiliary samples derived from the observations.
This allows for the use of any flexible function approximator such as neural networks or boosted decision trees, while mitigating the overfitting problem common to other flexible conditional density estimators.
In \cref{sec:related_literature}, we outline connections of our contribution to the existing literature.
We formally introduce our approach in \cref{sec:method} and state our theoretical result regarding the convergence of our estimator.
The proof can be found in \cref{app:proof_details}.
We propose \condensite\footnote{Pronounced \textit{kon-dahn-see-tay}, like the French word for ``density''.}, a method that implements our approach, and describe the algorithm and implementation in \cref{sec:implementation}.
We provide a proof of concept and analyze the influence of hyperparameters on synthetic data in \cref{sec:synthetic}.
In \cref{sec:real_world} we provide a detailed analysis of the performance of our method on a large population survey dataset, and an evaluation and comparison on a satellite imaging dataset.
We discuss our findings in \cref{sec:discussion} and provide our conclusion in \cref{sec:conclusion}.

\section{Related Literature}\label{sec:related_literature}

One fundamental approach to conditional density estimation is the parametrization of distributional families using either specific functions \parencite[e.g.\ linear as in][Chapter 14]{gelman2013bayesian}, or flexible function approximators \parencite[e.g.][]{bishop1994mixture}.
Parametric approaches can suffer from model misspecification and have limited flexibility to adapt to different local structures, especially in high-dimensional settings.
In our review of the literature, we therefore focus on nonparametric approaches.
Among these, we draw the following distinction with the goal of connecting our own contribution.
We first cover what we refer to as \textit{density-based} approaches, which estimate the conditional density as a combination of other densities.
Then, we give an overview of \textit{regression-based} approaches, which leverage regression for conditional density estimation.
Other than through conditional densities, conditional laws can also be characterized through conditional quantiles and conditional cumulative distribution functions, giving rise to quantile regression and distribution regression, which are closely related to conditional density estimation. For details on these tasks we refer to \cite{koenker2017handbook} and \cite{Kneib2023Rage}, respectively.

\paragraph{Density-based approaches.}
A conditional density $f_{Y|X}$ can be understood as the ratio of $f_{X,Y}(\cdot,\cdot)$, the joint density of target and covariates, and $f_X(\cdot)$, the density of the covariates.
A simple way of estimating it is by estimating each component separately, e.g.\ using kernel density estimators as in \cite[]{rosenblatt1969conditional}, and then performing division.
In high dimensions, the resulting sparsity of data points means that the density estimation may become unstable, for instance if the bandwidth is not selected suitably, which necessitates the use of complex bandwidth selection procedures \parencite[e.g.][]{hyndman1996estimating,hall2004cross}.
An alternative is proposed by \cite{Efromovich2007}, who use an orthogonal series estimator.
Similarly to the kernel-based methods, this method also struggles in high dimensions due to the curse of dimensionality. This prompts \cite{Efromovich2010dimension} to add adaptive dimension reduction to the orthogonal series approach, which helps in reducing the impact of irrelevant covariates, but the problem remains if many covariates are relevant.
\cite{izbicki2016nonparametric} develop and study an orthogonal series estimator that achieves good performance for many relevant covariates with low intrinsic dimensionality, but struggles with irrelevant covariates.
Another way of approaching conditional density estimation is to apply unconditional density estimation to observations with similar conditional densities.
\cite{cevid2020distributional} utilize random forests as nearest neighbors method \parencite[see][]{lin2006random} in order to construct weighted empirical distributions that can be smoothed into densities.
\cite{gao2022lincde} partition the covariate space using boosted decision trees and perform unconditional density estimation using Lindsey's method \parencite{lindsey1974comparison} for each partition.
Both of \cite{cevid2020distributional} and \cite{gao2022lincde} benefit from leveraging machine learning methods that perform well in high-dimensional settings.

\paragraph{Regression-based approaches.}
Viewing each observation $(X_i, Y_i)$ as a sample from the conditional law $f_{Y|X}(\cdot|X_i)$, finding the conditional density at a given point $X_i$ can be seen as a prediction problem.
Since there are usually few or no samples at the given point, utilizing regression requires some form of smoothing or localization.
The ``double kernel'' \parencite{hall1999methods} approach by \cite{fan1996estimation} lays the foundation for this idea.
It can be understood as performing a localized regression for inference at a given point $(x,y)$.
The targets are evaluations of a first kernel $K_1(Y_i-y)$ for every data point $1\leq i\leq n$, and the covariates are the corresponding $X_i$.
In the regression, the samples are weighted by another kernel evaluation $K_2(X_i-x)$.
The resulting prediction at $x$ gives the conditional density at $(x,y)$, and by scanning over the range of $Y$ one can construct the entire conditional density for $x$.
The semiparametric estimator of \cite{hjort1996locally} also weighs in $X$, but does not perform the same kind of target smoothing.
Instead of smoothing targets, \cite{munoz2011super} localize the dependent variable through binning, and estimate sequential conditional bin probabilities with a super learner.
\cite{sugiyama2010conditional} estimate the conditional density directly via a least-squares density ratio objective developed in \cite{kanamori2009least}.
This method struggles in high dimensions, so \cite{shiga2015direct} add a penalty on irrelevant covariates, but the problem remains for settings with many relevant covariates.
In order to address settings with many relevant and many irrelevant covariates, \cite{izbicki2017converting} propose an orthogonal series method that leans heavily on regression.
Whereas previous orthogonal series methods expand $f_{Y|X}$ in $X$ and $Y$, they expand only in $Y$ and estimate the coefficients as a function of $X$ through regression.
This enables the use of regression methods that perform well in high-dimensions.

\paragraph{Placing our contribution.}
Our approach builds upon ideas from the regression-based conditional density estimation literature. 
At its core is a single regression, akin to \cite{sugiyama2010conditional}. 
However, they target the conditional density directly, which limits the approach to low dimensions, since in high dimensions the data points are too sparse for the regression to perform well.
By contrast, we target an auxiliary approximation of the conditional density, which is defined using a probability kernel.
\cite{fan1996estimation} introduce the same kind of target, but perform a covariate-specific regression that is localized by a kernel similarity of the covariates, which does not work well in high dimensions for the same reason as the other kernel-based methods.
Thus, our approach -- implemented as the method condensité -- can be seen as a combination of the targets from \cite{fan1996estimation} with the regression framing in \cite{sugiyama2010conditional}, allowing it to overcome the limitations of either method alone.
Moreover, unlike these earlier works, \condensite uses powerful nonparametric regression methods that have since become widely available.
\cite{munoz2011super} and \cite{izbicki2017converting} also use such regression methods, but in an indirect fashion that offloads much of the complexity to binning and basis expansion, respectively.
By utilizing regression directly, \condensite opens up the possibility of transferring the success of high-dimensional regression methods such as those in \cite[][]{ke2017lightgbm,paszke2019pytorch} to high-dimensional conditional density estimation.

\section{Method}\label{sec:method}

This section contains two parts. First, we provide an overview of the approach underlying our method, including a formal description of the setting and key idea, and a numerical illustration of how the auxiliary targets for the regression problem are derived.
Second, we outline in what sense our estimator converges, illustrate the intermediate steps of the transformation into a regression problem, and state our theoretical main result.

\subsection{Overview}

\paragraph{Setting and objective.} Suppose there are independent observations
$(X_1, Y_1), \ldots, (X_n,Y_n)$ following the law $P^\star$ on $\X\times[0,1]$ of a generic couple of random variables $(X, Y)$.
We assume $P^\star$ has a joint density $f_{X,Y}^\star$ with respect to a dominating product measure $\mu\otimes\Unif[0,1]$ for some measure $\mu$ on a measurable space $(\X,\mathcal{B})$.
Our goal is to learn the collection of conditional densities of $Y$ given $X$, $\{y\mapsto \fstar(y|x)\colon x\in\Supp(\PX)\}$, where $\PX$ is the marginal law of $X$ under $P^\star$ (note that here, and moving forward, we write conditional laws without the subscript $Y|X$ for notational ease).

\paragraph{Key idea.}
In order to transform conditional density estimation into a regression task, we start by choosing an approximate identity $\{K_h\colon h>0\}$.
We will rely on the properties of the approximate identity, as stated in \cref{app:ApproximateIdentity}, to show that the transformation into a regression task recovers $\fstar$ as $h$ goes to zero.
For instance, $K_h$ could be the density of the centered Gaussian law with variance $h^2$.

For our estimator, we choose a closed and convex class $\F$ of functions $\X\times[0,1]\to\mathbb{R}_+, (x,y) \mapsto f(y|x)$. We ensure that there exists a constant $c>0$ such that, for all $h>0$, $h\|K_h\|_\infty\leq c$ and $h\|f\|_\infty\leq c$ for every $f\in\F$ (hence the $h$ in the index of $\F$).

We set up our regression as follows. For each $1\leq i\leq n$, independently of the observations, we draw $M$ auxiliary samples $Y_{i1}^\prime,\ldots, Y_{iM}^\prime$ independently from $\Unif[0,1]$ and compute every $K_h(Y_i - Y'_{im})$ -- the closer $Y_{im}^\prime$ is to $Y_i$, the larger the result.
We then find the best $f\in\F$ to approximate the $K_h(Y_i-Y_{im}^\prime)$ by $f(Y_{im}^\prime|X_i)$ across all data points with $1\leq i\leq n$ and $1\leq m\leq M$ with respect to the least-squares criterion
\begin{align}
    \estimator \coloneq \underset{f\in\F}{\arg\min} \sum_{i=1}^n\sum_{m=1}^M\left[K_h(Y_{i}-Y'_{im}) - f(Y'_{im}| X_i)\right]^2.
    \label{eq:key_idea_empirical}
\end{align}
We analyze in what sense $\estimator$ approximates the true conditional density $\fstar$.
For this we rely on the theoretical counterpart to \cref{eq:key_idea_empirical} given by
\begin{align}
    \projection \coloneq \argmin_{f\in\F}\int_{x\in\X}\int_{y\in[0,1]}\int_{y'\in[0,1]^M}\sum_{m=1}^M\left[K_h(y-y'_m)-f(y'_m| x)\right]^2f_{X,Y}^\star(x,y)\prod_{m=1}^M \dd y_m^\prime\;\dd y\;\mu(\dd x).
    \label{eq:key_idea_theoretical}
\end{align}
Note that $M$ affects the estimation of $\estimator$ in \cref{eq:key_idea_empirical}, but the inner integral in \cref{eq:key_idea_theoretical} is the same regardless of the value of the hyperparameter $M$, so $\projection$ is invariant to the choice of $M$.

\paragraph{Illustration.}
For illustration, we provide a simple example using the two-dimensional data-generating mechanism with $X\sim\N(0,1)$ and $Y|X \sim \N(X, 0.25+X^2)$.
\cref{fig:illustration} shows how the joint law of $(X,Y)$ gives rise to our regression targets (note that we rescale $Y$ to $[0,1]$ by min-max scaling).
We sample $n=5000$ observations independently, use $M=100$ auxiliary samples per observation, and set a sharpness of $h=0.03$. 
The joint density shown in \cref{fig:joint_density} is concentrated around the point $(0,0)$, and one can recognize the linear slope and heteroskedasticity of the conditional density.
In \cref{fig:auxiliary_targets} one can see how the transformations $K_h(Y_i-Y_{im}')$ map samples from the joint law into our regression targets mimicking the conditional law. The concentration around $(0, 0)$ is gone, and instead each vertical slice along the x-axis of the right-hand plot approximates a renormalized version of the corresponding slice of the left-hand plot.
\begin{figure}[H]
    \centering
    \begin{subfigure}{.49\linewidth}
        \centering
        \includegraphics[width=\linewidth]{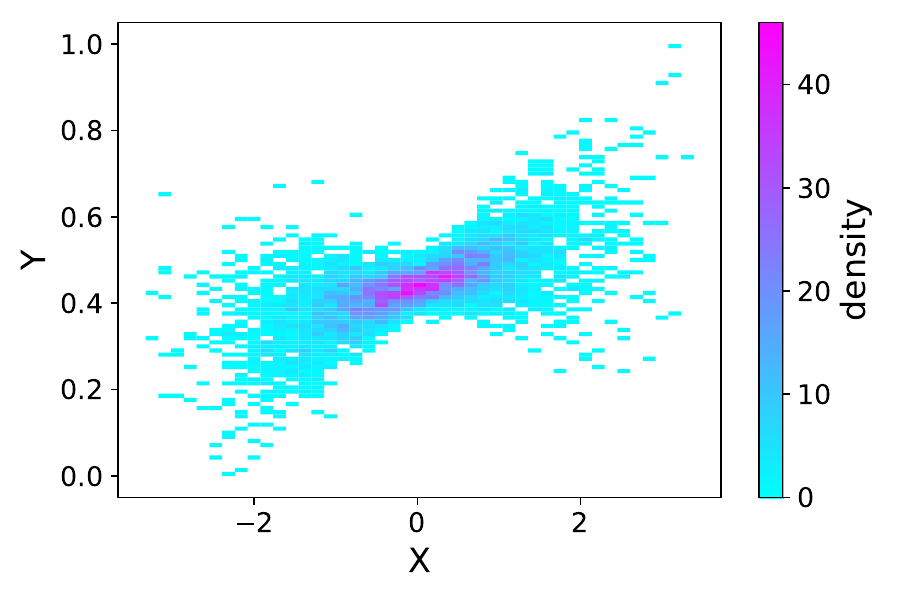}
        \caption{Empirical joint density of $(X,Y)$.}
        \label{fig:joint_density}
    \end{subfigure}
    \begin{subfigure}{.49\linewidth}
        \centering
        \includegraphics[width=\linewidth]{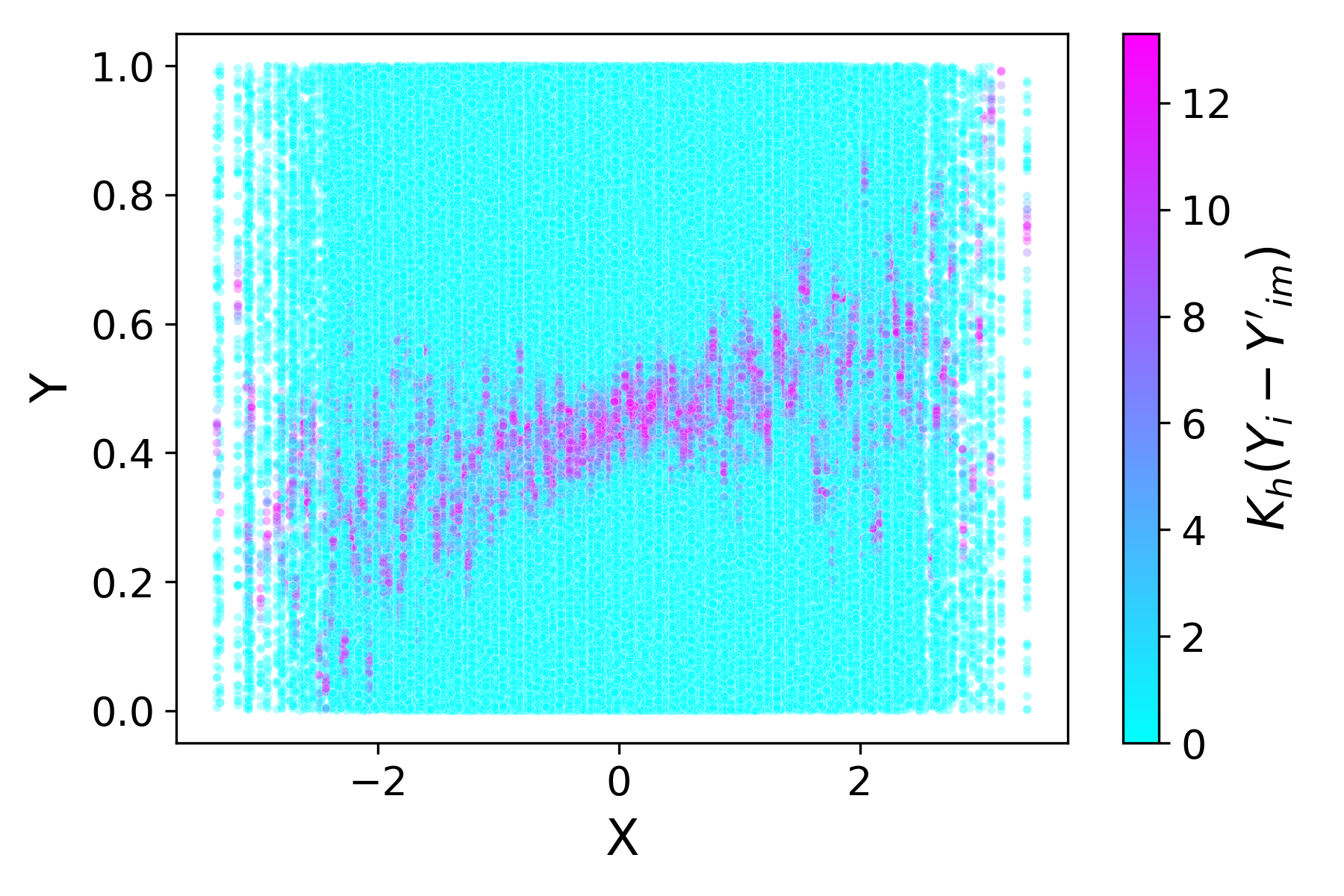}
        \caption{Scatterplot of the regression targets.}
        \label{fig:auxiliary_targets}
    \end{subfigure}
    \caption{Illustration of the key idea.}
    \label{fig:illustration}
\end{figure}

\subsection{Theoretical Main Result}
What follows is a formal definition of the steps involved in defining our estimator, and the result of our theoretical analysis of its convergence.
For reference, we provide glossaries of central objects in \cref{app:glossary}. 
Definitions of random variables are given in \cref{tab:var_glossary},  functions in \cref{tab:f_glossary}, and laws in \cref{tab:p_glossary}.
The full proofs can be found in \cref{app:proof_details,app:ApproximateIdentity}.
We use the following notational convention: for $P$ a measure on some measurable space $(S, \mathcal{B})$ and $\varphi\colon S\to\mathbb{R}$ a measurable function, $P\varphi$ denotes the integral $\int \varphi dP$; in particular, $P\varphi^2 = \|\varphi\|^2_{L^2(P)}$.

Let $\ell$ be the loss function mapping any $f\in\F$ to the function $\ell[f]$ characterized by
\begin{align*}
    \ell[f](x,\bar{y}',\bar{z}) = \frac{1}{M}\sum_{m=1}^M\left(f(y_{m}'| x)-z_{m}\right)^2,
\end{align*} 
with $x\in\X$, $\bar{y}'=(y'_1,\dots,y'_M)\in[0,1]^M$, and $\bar{z}=(z_1,\dots,z_M)\in\mathbb{R}_+^M$.
This corresponds to $1/M$ times the inner summand of \cref{eq:key_idea_empirical} and the integrand of \cref{eq:key_idea_theoretical}.

Let $\Ph$ be the joint law of $(X,\bar{Y}',\bar{Z}_h)$, where 
$(X,Y)$ is drawn from $\P$, 
$\bar{Y}'=(Y'_1,\ldots,Y'_M)$ is drawn from $\left(\Unif[0,1]\right)^{\otimes M}$ independently of $(X,Y)$, and 
$\bar{Z}_{h}=(Z_{h,1},\ldots,Z_{h,M})$ with $Z_{h,m}=K_h(Y-\bar{Y}'_m)$ for each $1\leq m\leq M$.
We denote the minimizer of the risk function induced by $\ell$ under the law $\Q=\PX\otimes\Unif[0,1]$ over all square integrable functions as
\begin{align*}
    \transformation \coloneq \argmin_{f\in L^2(\Q)} \left\{\Ph \ell[f]\right\}.
\end{align*}
Since $\{K_h\colon h>0\}$ is an approximate identity, for $h$ approaching zero the convolution $K_h*\varphi$ converges to $\varphi$ in $L^p(\mathbb{R})$ for any function $\varphi\in L^p(\mathbb{R})$ with $p\geq1$, and converges pointwise to $\varphi$ for any bounded and uniformly continuous $\varphi$ (see \cref{prop:approx_iden} in \cref{app:ApproximateIdentity}).
By \cref{eq:transformation_h_limit} it follows that $\transformation$ converges to $\fstar$ as $h$ goes to zero \parencite[cf.][Equation 2.1]{fan1996estimation}.

The function $f^\star_{h,\F}$ defined by \cref{eq:key_idea_theoretical} is the projection of $f_h^\star$ onto $\F$ and can be written as
\begin{align*}
    \projection \coloneq \argmin_{g\in\F}\left\{Q^\star(g-f_h^\star)^2\right\}.
\end{align*}
This is the function approximated by the empirical risk minimizer defined in \cref{eq:key_idea_empirical} and below as
\begin{align*}
    \estimator \coloneq \argmin_{f\in\F}\left\{\Phn \ell[f]\right\},
\end{align*}
where $\Phn$ is the empirical law that puts mass $1/n$ on every $(X_i,\bar{Y}_i^\prime,\bar{Z}_{hi})$.
A schematic of the relationship between $\fstar,\transformation,\projection,\estimator$ can be seen in \cref{fig:targets}.
\vskip1em

We assess the proximity of $f\in\F$ to $\transformation$ through the excess risk induced by loss $\ell$ with respect to $\projection$ which we denote as
\begin{align}
    \ER_h(f) &\coloneq \Q(f-\transformation)^2- \inf_{g\in\F} \Q(g-\transformation)^2\nonumber\\
    &=\Q(f-\transformation)^2-\Q(\projection-\transformation)^2\nonumber\\
    &=\Ph(\ell[f]-\ell[\projection]).
    \label{eq:ER_loss}
\end{align}
\noindent
\begin{figure}[H]
    \centering
    \begin{tikzpicture}
        \node (f) at (-0.5,2.3) {$\fstar$};
        \node (fh) at (1,2) {$\transformation$};
        \draw[->,dashed] (f) to[bend left] (fh);
        \node[align=center] (annot) at (1.4,3) {transformation into\\a regression task};
        \node[align=center] (annot2) at (1.9,1.2) {projection};
        \node (proj) at (1,0) {$\projection$};
        \draw (0,-.5) ellipse (3cm and 1.3cm);
        \node (F) at (3.5,-.5) {$\F$};
        \draw[->] (fh) -- (proj);
        \node (est) at (-1.5, -.8) {$\estimator$};
        \node[align=center] (annot3) at (.6,-.75) {approximation};
        \draw[->,decorate,decoration={snake,amplitude=1mm,segment length=5mm}] (est) -- (proj);
        \draw[gray] (-4,-2.5) rectangle (4.8,4); %
    \end{tikzpicture}
    \caption{Illustration of the transformation of the conditional density estimation problem into a regression task. We use an approximate identity with bandwidth $h$ to define the regression task, choose a function class $\F$ for the regression, and optimize the fit on a set of observations.}
    \label{fig:targets}
\end{figure}
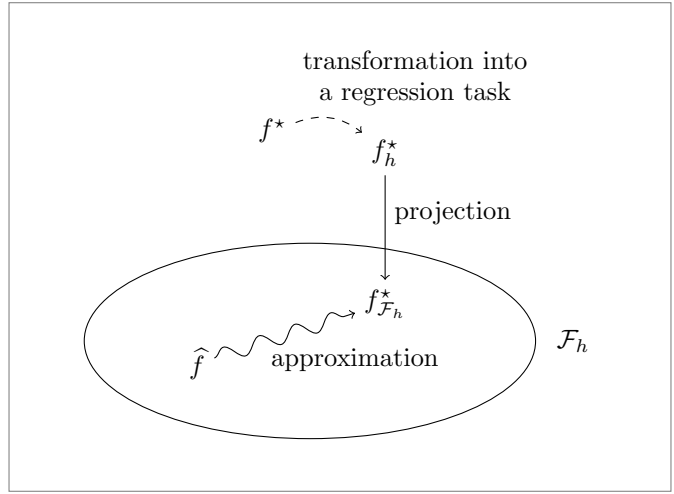

\begin{assumption}
Suppose that $\F$ is chosen such that its complexity in terms of covering numbers is controlled by the following condition \parencite[][Section 2.5.1]{Vaart1996}: there exist a measurable envelope $F_h$, constants $A>0$ and $\rho \in [0,1)$ such that for every probability measure $Q$ on $\X\times[0,1]$ with finite support, for all $\varepsilon>0$, 
    \begin{align}
        \log N(\varepsilon,\F,L^2(Q)) \leq \left(\frac{A\|F_h\|_{L^2(Q)}}{\varepsilon}\right)^{2\rho}.
        \label{eq:covering_number_condition}
    \end{align}
    \label{ass:cov_num}
\end{assumption}
\noindent
This implies the finiteness of the uniform entropy integral $\sup_Q\int_0^1 \sqrt{1 + \log N(\varepsilon,\F,L^2(Q))}\dd\epsilon$.
The excess risk of our estimator, as defined in \cref{eq:ER_loss}, is bounded by the following concentration inequality.
\begin{theorem}
    [Excess risk bound]
    Under \cref{ass:cov_num}, there exists a constant $C>0$ such that, for all tuning parameters $M \geq 1$ and $h > 0$, for all $t>0$,
    \begin{align}
        \Ph\left(\ERh(\estimator) \geq C\left[
            \frac{c^2 2^{1-\rho}}{h^2\sqrt{n}(1-\rho)} 
            + \frac{t}{n}\left(\sqrt{\left(\frac{2c}{h}\right)^2+\frac{1}{8}}-\frac{2c}{h}\right)^{-2}\right]\right)\leq e^{1-t}.
        \label{eq:theorem}
    \end{align}
    \label{thm:1}
\end{theorem}
\noindent
The proof of the theorem is an adaptation of that of \textcite[Theorem 4.3]{koltchinskii2011oracle} using \textcite[Corollary 1]{maurer2016vector}, and is deferred to \cref{app:proof_details}.
For simplicity, we use the constant envelope function $c/h$ which, without loss of generality, upper bounds the envelope $F_h$ used in \cref{ass:cov_num}. The result is improved when using $F_h$ directly.
A function with a larger (constant) envelope may lead to $\projection$ being closer to $\transformation$, but would also lead to a looser bound.
Observe further that, as $h$ goes to zero, the second summand is of the form $\frac{tc^2}{nh^2}(1024+o(1))$ and goes to infinity.
Yet, $h$ must go to zero for $\transformation$ to converge to $\fstar$ by \cref{eq:transformation_h_limit}, so there is a bias-variance trade-off.
Finally, note that \cref{thm:1} does not capture an effect of the choice of the hyperparameter $M$. The forthcoming simulation study in \cref{sec:h_M_interplay} investigates its impact empirically.

\section{Implementation}\label{sec:implementation}
\paragraph{Evaluation.}
We evaluate the discrepancy between our estimator $\estimator$ and the true conditional density $\fstar$ using the integrated squared error (ISE).
The ISE is given by
\begin{align}
    \int\!\!\int \left(\estimator(y|x)-\fstar(y|x)\right)^2 \PX(\dd x) \dd y 
    = &\int\!\!\int \estimator^2(y|x) \PX(\dd x)\dd y 
    - 2\int\!\!\int \estimator(y|x)\fstar(y|x) \PX(\dd x)\dd y+ C,
    \label{eq:ISE}
\end{align}
where $C$ is a constant that does not depend on $\estimator$.
We estimate the ISE up to $C$ (hence the negative values in our experiments).
In our experiments we approximate the ISE over a grid using trapezoidal integration.

\paragraph{Algorithm.}
We provide a modular implementation that allows for using a variety of predictors for conditional density estimation.
It consists of the following two parts.
\begin{enumerate}
    \item \texttt{Condensite}
    \item \texttt{CondensitePredictor}
\end{enumerate}
The \texttt{Condensite} class implements the functionality for transforming samples from the joint law into targets for conditional density estimation, and allows for fitting a \texttt{CondensitePredictor}. In doing so, it takes care of appropriate variable scaling and evaluates the fit on a validation set.
\texttt{CondensitePredictor} is an abstract class that acts as an interface that predictors passed to \texttt{Condensite} have to implement.
In principle, any regression model can serve as the basis for a \texttt{CondensitePredictor}.
This modularity allows our approach to take advantage of a wide range of different implementations. 
Alongside our Python implementation we provide examples using predictors from the \texttt{PyTorch} \parencite{paszke2019pytorch}, \texttt{LightGBM} \parencite{ke2017lightgbm}, and \texttt{scikit-learn} \parencite{pedregosa2011scikit} libraries.
The following pseudocode algorithm captures the essential part of the data fitting functionality of \texttt{Condensite}.

\SetAlCapSkip{.5em}
\SetAlgoSkip{bigskip}
\SetKwFunction{KwReturn}{return}      %
\SetFuncSty{textbf}                   %
\SetNlSty{textbf}{\large}{.}          %
\setlength{\algomargin}{1.8em}        %
\SetKwComment{comment}{\hspace{1em}// }{}
\SetCommentSty{mycommfont}
\newcommand{\mycommfont}[1]{\textcolor{blue}{#1}}
\LinesNumbered
\DontPrintSemicolon
\SetNlSty{textbf}{\small}{.}
\begin{algorithm}[H]
    \setlength{\baselineskip}{1.2\baselineskip}  %
    \KwIn{data $(X_i, Y_i)_{1\leq i\leq n}$,
          number of auxiliary samples $M$,
          sharpness $h$.}
    \KwOut{estimator $\estimator$, validation ISE.}
    Split data into training and validation sets.\\
    Repeat training samples of $X$ a total of $M$ times.\\
    Min-max-scale the training samples of $Y$ to $[0,1]$.\\
    Generate auxiliary samples $\bar{Y}'$ by sampling $M$ points uniformly in $[0,1]$ per training data point.\\
    Compute the targets as $K_h(Y_{i}-Y_{im}')_{}$ per training and auxiliary sample point.\\
    Column-stack the repeated training samples of $X$ and auxiliary samples $\bar{Y}'$ into a feature matrix.\\
    Standardize each feature and the targets.\comment{for training stability}
    Find $\estimator$ by minimizing \cref{eq:key_idea_empirical}.\\
    Compute ISE of $\estimator$ on validation data.\comment{includes inverse scaling of predicted targets}
    \textbf{Return:} $\estimator$, ISE.
    \caption{Fitting \condensite.}
\end{algorithm}

\paragraph{Post-processing.}
To make sure we obtain a density, we 
set $\estimator(\cdot|\cdot) := \max\{0, \estimator(\cdot|\cdot)\}$, and then redefine $\hat{f}(\cdot|\cdot) := \estimator(\cdot|\cdot) / \int_a^b\hat{f}(y|\cdot)\dd y$.
\section{Proof of Concept and Hyperparameter Analysis}\label{sec:synthetic}
We test \condensite on synthetic data from the following three synthetic data-generating mechanisms inspired by \textcite[Section 4.1]{izbicki2016nonparametric}.
The primary goal of our synthetic data experiments is to provide a proof of concept of \condensite in a simple and controlled setting, and to examine the impact of the hyperparameters $h$ and $M$.
For each mechanism, we use $20$ independently standard Gaussian distributed covariates $X$.
\begin{enumerate}
    \item \textbf{Single relevant covariate:} the single covariate $X^{(1)}$  determines the conditional density via
    $Y|X \sim \N(X^{(1)}, 0.25+(X^{(1)})^2)$.
    \item \textbf{Data on manifold:} the conditional density depends on the angle $\theta$ mapped to $[0,2\pi]$ between $X^{(1)}$ and $X^{(2)}$ via
    $Y|X \sim \N(\theta, 0.5)$.
    \item \textbf{Non-sparse data:} the conditional density depends on all covariates via $Y|X \sim \N\left(\operatorname{mean}(X), 0.5\right)$.
    \label{enum:settings}
\end{enumerate}
Note that the first setting is heteroskedastic to make it more challenging; the relevant covariate affects $Y$ as shown in our illustration in \cref{fig:illustration}.
We evaluate two different versions of \condensite:
\begin{itemize}
    \item \condensiteNN: a version using a neural network as predictor.
    \item \condensiteTree: a version using a gradient boosted decision tree as predictor.
\end{itemize}

\subsection{Proof of Concept}\label{sec:PoC}

We train the methods on $10000$ samples and evaluate on another $1000$ samples for each data-generating mechanism.
We choose $M=100$ and $h=10^{-2}$ for the \condensite methods; for details on the hyperparameters, see \cref{app:hyper_condensite}.
For comparison, we run the following methods from the literature:
\begin{itemize}
    \item\lincde \parencite{gao2022lincde},
    \item\drf \parencite{cevid2020distributional},
    \item\flexcode \parencite{izbicki2016nonparametric}, 
    \item\condensier \parencite[based on][]{munoz2011super}.
\end{itemize}
\lincde and \drf employ decision tree variants that have proven successful in various regression and classification contexts.
\flexcode and \condensier utilize flexible regression models. 
We choose the decision tree predictor provided in the \flexcode implementation and call this version \flexcodeTree, for \condensier we keep the inbuilt default predictor.
All four methods promise good performance in high dimensions.
We provide details on the implementation and hyperparameters in \cref{app:hyper_literature}.

\paragraph{Proof of concept results.} The results in terms of ISE (lower is better) are shown in the table below.
The estimates by \condensite are on par with those of the other methods from the literature. The landmark (specific covariates) analysis in \cref{app:landmarks} shows that all methods successfully learn a density close to the ground truth, with \condensiteNN and \lincde providing the smoothest fits.
Note that the same hyperparameter configuration is used across datasets for each method, including the \condensite variants. This suggests that \condensite can achieve good out-of-the-box results using generic choices for the hyperparameters $M$ and $h$. The following analysis explores their influence in greater detail.
\begin{table}[H]
    \small
    \centering
    \begin{tabular}{lccc}
\toprule
 & \textbf{Single relevant covariate} & \textbf{Data on manifold} & \textbf{Non-sparse data} \\
\midrule
\textit{condensier} & $-0.2865$ & $-0.3561$ & $-0.2651$ \\
\textit{condensité (NN)} & $-0.3311$ & $-0.3853$ & $-0.2800$ \\
\textit{FlexCode (knn)} & $-0.2561$ & $-0.2637$ & $-0.2694$ \\
\textit{FlexCode (tree)} & $-0.3068$ & $-0.3701$ & $-0.2584$ \\
\textit{DRF} & $-0.3453$ & $-0.3973$ & $-0.2815$ \\
\textit{LinCDE} & $-0.3364$ & $-0.3964$ & $-0.2832$ \\
\bottomrule
\end{tabular}

\end{table}

\subsection{Hyperparameter Analysis: The Effect of \texorpdfstring{$M$}{\textit{M}} and Its Interaction With \texorpdfstring{$h$}{\textit{h}}}\label{sec:h_M_interplay}
The hyperparameters $M$ and $h$ are at the core of our method. We investigate their interplay on our synthetic data.
We use the same hyperparameters as in the proof of concept experiment (see \cref{app:hyper_condensite}), which yield good results for $h=10^{-2}$ and $M=100$ on all of our data-generating mechanisms.
To investigate the role of $M$ and its interplay with $h$, we run a grid search over $h\in\{10^{-4},10^{-3},10^{-2},10^{-1}\}$ and $M\in\{1, 10, 20, 40, 80, 160, 320\}$ for each mechanism.
\cref{fig:h_M_interplay_NN} shows the results in terms of ISE for \condensiteNN on each of the data-generating mechanisms.
Note that we limit the training to $20$ epochs, 
so it is possible that some of the results would improve with further training and our findings here may reflect the speed of convergence as well as the best attainable performance.
\begin{figure}[H]
    \centering
    \begin{subfigure}{.32\linewidth}
        \centering
        \includegraphics[width=\linewidth]{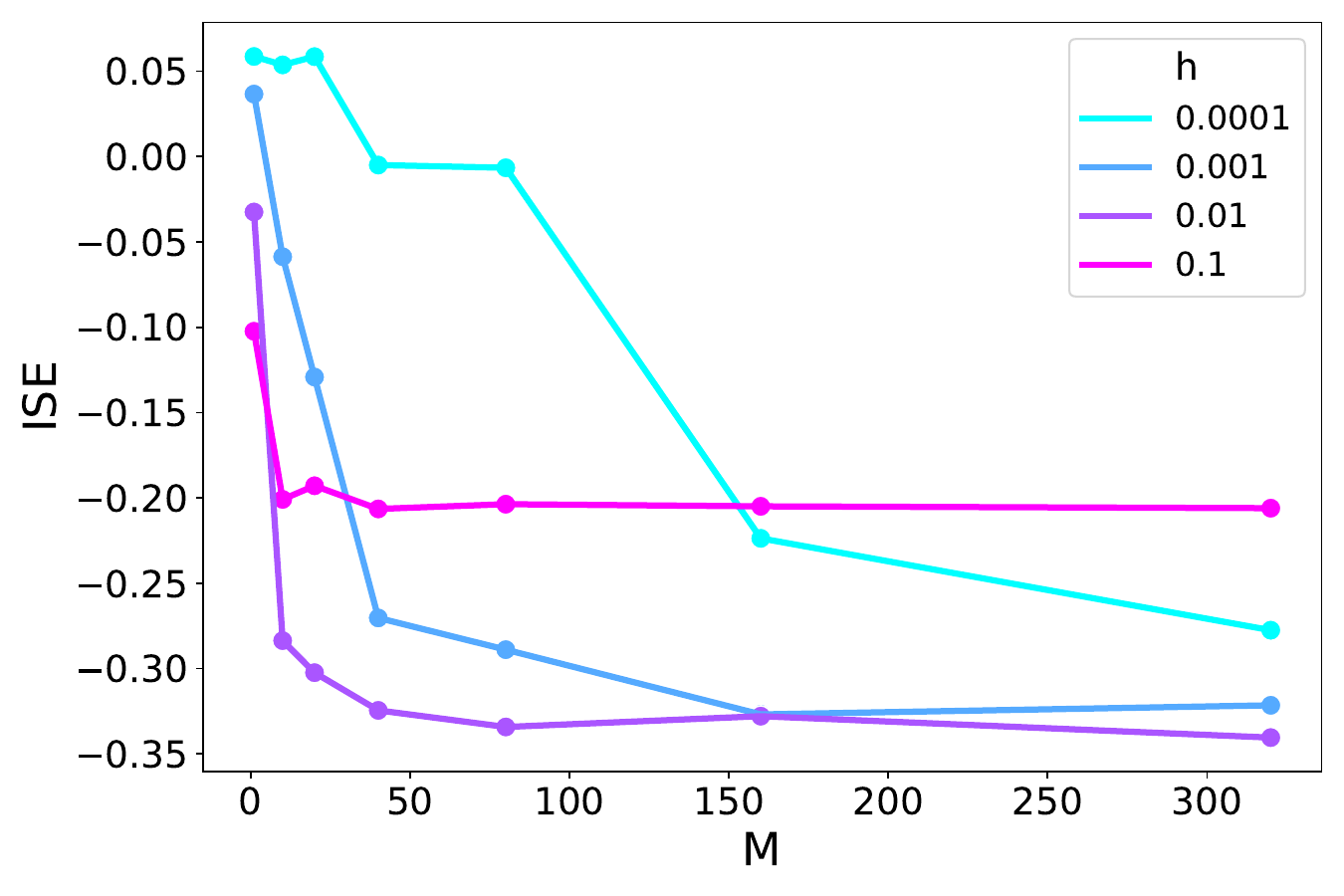}
        \caption{Single relevant covariate.}
    \end{subfigure}
    \begin{subfigure}{.32\linewidth}
        \centering
        \includegraphics[width=\linewidth]{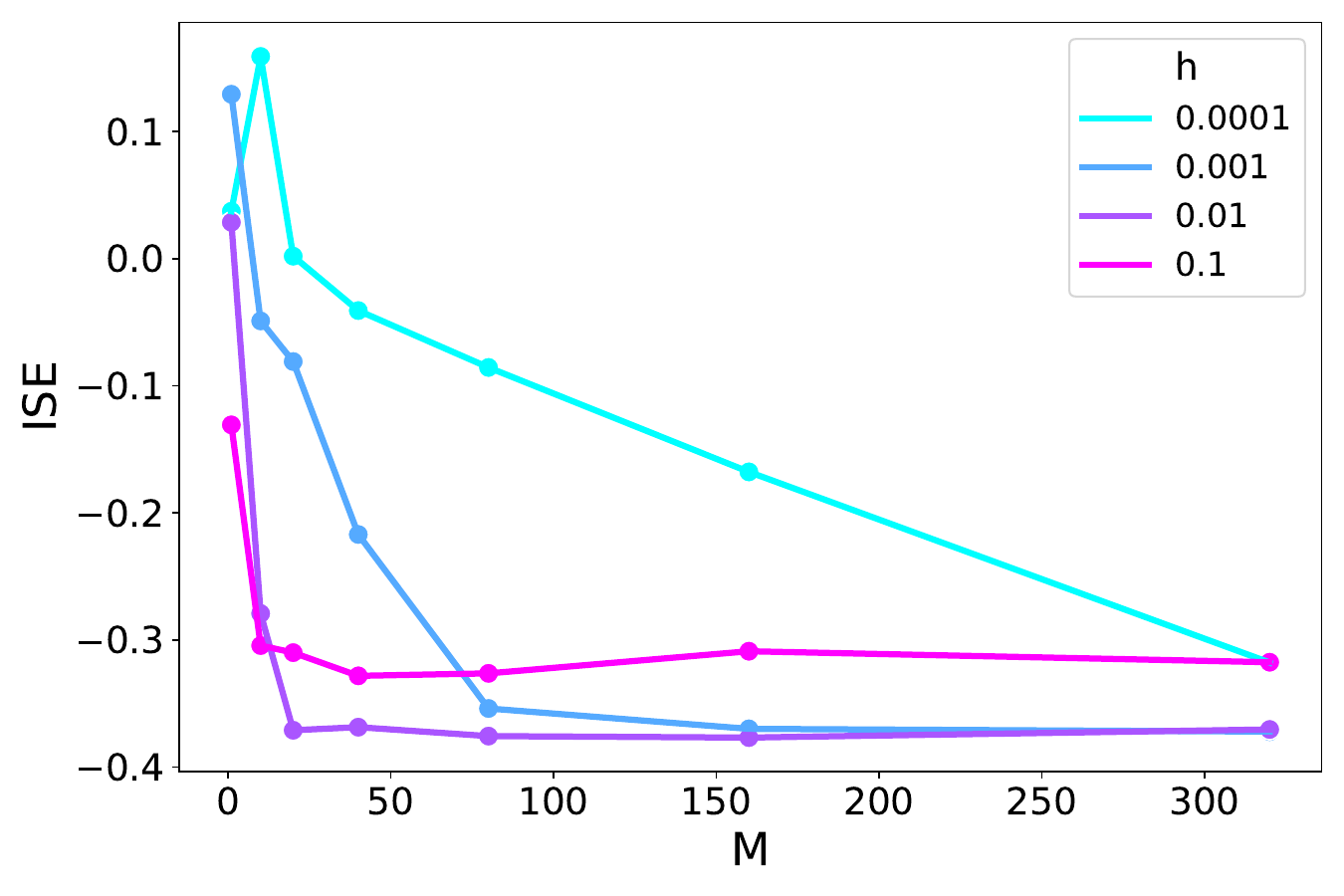}
        \caption{Data on manifold.}
    \end{subfigure}
    \begin{subfigure}{.32\linewidth}
        \centering
        \includegraphics[width=\linewidth]{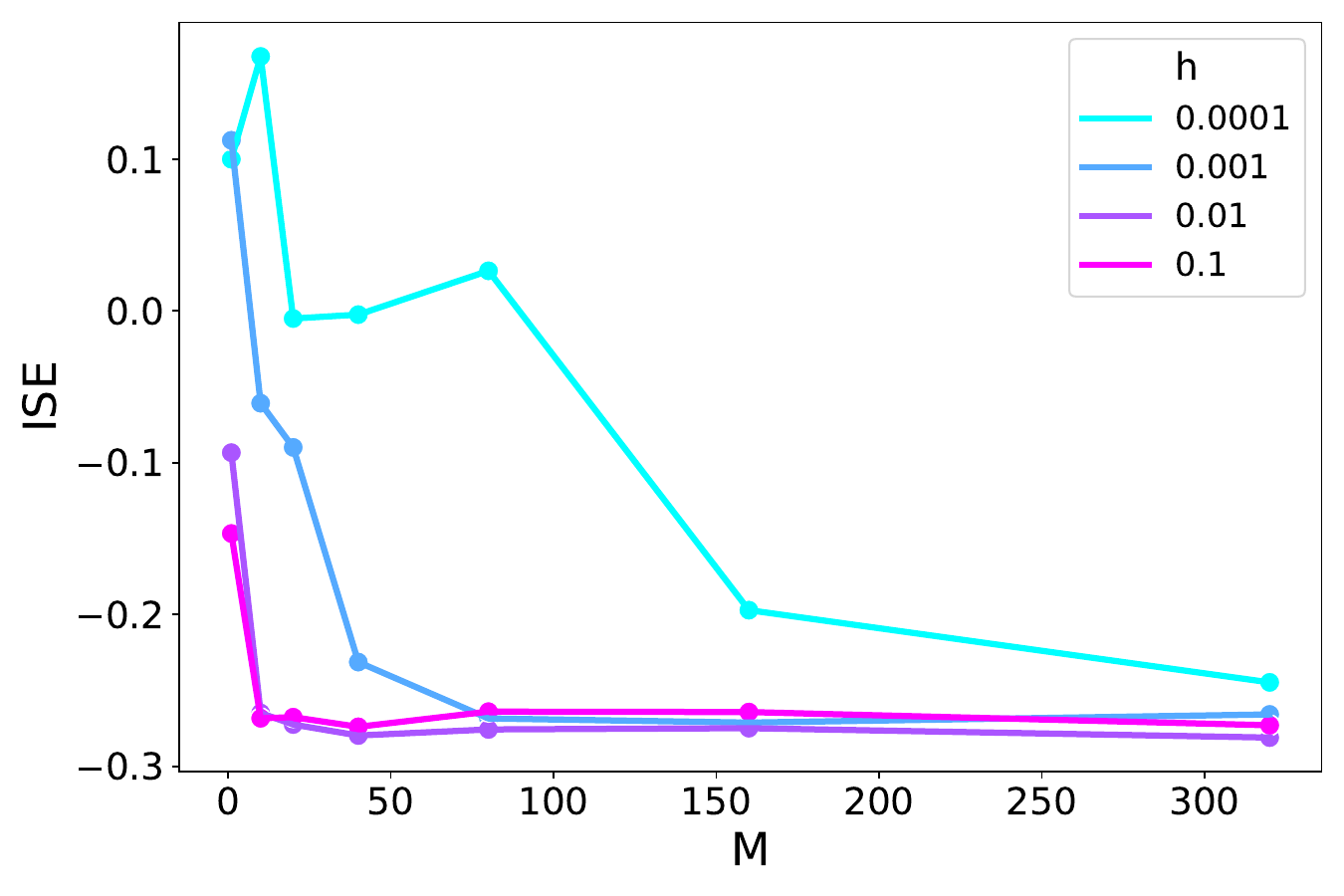}
        \caption{Non-sparse data.}
    \end{subfigure}
    \caption{ISE for \condensiteNN with different $M$ and $h$ on the synthetic data-generating mechanisms.}
    \label{fig:h_M_interplay_NN}
\end{figure}
\noindent
Across the three settings we find that larger values of $M$ tend to give better results, and that lower $h$ require larger $M$ to work well. 
For well-chosen $h$, here $0.01$ and $0.001$, even a moderate number of $M$ seems to improve the performance substantially.
If $h$ is chosen too high, the value of $M$ seems to make little difference, and the method does not perform as well.
Since $h$ must go to zero for the transformation into a regression to recover the true conditional density, this is not surprising.
For $M$ large enough, performance stabilizes at different levels depending on $h$.
The results for \condensiteTree are shown in \cref{fig:h_M_interplay_Tree} in \cref{app:h_M_interplay}.
They exhibit the same qualitative trends, although the performance benefit of larger $M$ materializes and also levels out faster than for \condensiteNN.
In our following experiments we use $h=10^{-2}$ and $M=100$ throughout for both of our estimators, since this combination works well on our synthetic data and appears to provide sufficient flexibility in $\F$ without requiring an excessively high number of auxiliary samples.

\section{Evaluation on Real-World Data}\label{sec:real_world}

To assess the promise of our approach under realistic conditions, we evaluate the performance of the \condensite methods and compare them to other methods from the literature on two challenging real-world datasets.

\subsection{IPUMS Current Population Survey Data}\label{sec:IPUMS}

We evaluate \condensiteNN, \condensiteTree, and the other methods from the literature on a dataset compiled from the IPUMS Current Population Survey (CPS) database\footnote{https://cps.ipums.org/cps/} \parencite{ipumsCPS2024}.
The IPUMS-CPS is a monthly household survey of more than 65000 households from the United States of America. 
It collects data for social science research and has been conducted since the year 1962.
It is conducted jointly by the US Census Bureau and the US Bureau of Labor Statistics.
The IPUMS database has been used particularly widely in economics research. 
For a recent example, see \cite{blanchet2022real}, who propose a real-time measure of inequality by estimating income distributions conditional on different economic and demographic variables.
We use a selection of variables measuring geographic, demographic, work-related, and education-related characteristics of individuals.
Given these covariates, we estimate the conditional density of yearly total personal income.
The hyperparameters used for the different methods are listed in \cref{app:hyper}, and we describe our data preprocessing steps in \cref{app:ipums_pp}.

Our evaluation consists of a quantitative and a qualitative part.
First, we compare the ISE of our methods and the other methods from the literature introduced previously in \cref{sec:synthetic}.
Second, we visually inspect the estimates of the different methods at two landmarks that correspond to typical questions in the economic literature on income inequality.
Since the true conditional density is unknown, we manually construct a local empirical density as a reference point for the plausibility of the estimates and discuss the consistency of the estimates with established findings in the economics literature.

\subsubsection{Empirical Comparison}

Our CPS dataset comprises $113{,}104$ observations of $26$ covariates, of which $6$ are multi-valued, with the remaining ones being binary.
We swap \flexcodeKNN for \flexcodeTree, a tree-based version better suited to tabular data. For the other methods we keep the same hyperparameters as before (see \cref{app:hyper} for details on the hyperparameters). 
We train on $80\%$ of the data and evaluate the ISE on the test set given by the remaining $20\%$.
For \drf, we limit the number of training samples to 50000 due to the large memory requirements of the implementation.
The results in terms of ISE (lower is better) can be seen in \cref{tab:IPUMS_ISE}.
\begin{table}[H]
    \centering
    \begin{tabular}{lc}
\toprule
 & \textbf{CPS} \\
\midrule
\textit{condensité (NN)} & $-0.0414$ \\
\textit{condensité (tree)} & $-0.0241$ \\
\textit{FlexCode (tree)} & $-0.0389$ \\
\textit{LinCDE} & $-0.0297$ \\
\textit{DRF} & $-0.0317$ \\
\textit{condensier} & $-0.0249$ \\
\bottomrule
\end{tabular}

    \caption{ISE results on the IPUMS-CPS dataset.}
    \label{tab:IPUMS_ISE}
\end{table}
\noindent
We find that \condensiteNN achieves the best result, with \flexcodeTree a close second.
\drf also performs well, especially given that it has been trained only on about half of the dataset.
\lincde is still close to \drf, but the other methods from the literature and \condensiteTree perform substantially worse. 
The difference between \condensiteNN  and \condensiteTree points to the choice of regressor as a decisive factor in our approach.
The \condensite methods are the only ones in this group that use regression directly for conditional density estimation and thus depend strongly on the inductive biases of different regression methods.
As can be seen, a well-chosen regressor allows for highly competitive performance.
Although \condensiteNN and \flexcodeTree appear to perform best, it is important to note that all methods offer a high degree of flexibility and different architecture choices, hyperparameters, or preprocessing steps may affect the results.
Moreover, we caution that good summary performance may not suffice for an estimator to be useful in practice.
In real-world applications, high-variance fits and artifacts may distort downstream analyses.
The following landmark analyses therefore explore qualitative aspects of the results.

\subsubsection{Landmark Analysis for Realistic Use-Cases}\label{sec:ipums_landmarks}

We choose two landmarks corresponding to realistic use-cases in the study of economic income inequality, and visualize the estimates of the different methods as well as a local empirical density for comparison.
In this section we focus on the results of the \condensite methods, and of \flexcodeTree as the second-best performing method (see \cref{tab:IPUMS_ISE}).
A visualization and analysis of the remaining methods can be found in \cref{app:ipums_landmarks}.

\paragraph{Skill premium.}
We inspect the conditional income densities estimated by our methods for two landmarks that differ only in years of education.
The income gap that arises between people with similar characteristics but different education is referred to as \textit{skill premium} and is traditionally analyzed by comparing mean wages between broadly defined groups of similar education \parencite[see e.g.][]{autor2008trends,Acemoglu2011skills}. 
However, analyzing wage distributions for specific subgroups and beyond simple summary statistics may be necessary to gain a more complete picture \parencite{firpo2018decomposing}. 
Conditional density estimation allows for doing so on a fine-grained level.
We choose the characteristics in the table below and vary the education level (see boldfaced row).
For reference we construct a local empirical density by including observations within $10$ years of age and within $10$ weekly work hours.
This yields $115$ observations for the $12$ year education landmark, and $175$ observations for the $16$ year education landmark.
\begin{table}[H]
    \centering
    \begin{tabular}{ll}
        \textit{characteristic} & \textit{landmark}\\
        \hline
        personal & $40$ years, male, no children, white (race)\\
        geography & metropolitan ($5+$ million)\\
        nativity & US-born (self and parents)\\
        work hours & $40$ (weekly, constant), private sector wage or salary\\
        \textbf{education} & $\mathbf{12}$ \textbf{years} $\boldsymbol{|}$ $\mathbf{16}$ \textbf{years}
    \end{tabular}
\end{table}
\noindent
\cref{fig:skill_prem} shows the methods' estimates and the local empirical density.
Overall, the methods' estimates are closely aligned with the local empirical density and show a strong and heterogeneous skill premium. 
For more years of education total incomes are not only higher, they are also substantially more dispersed.
This trend is in line with the long-term developments of increasing skill premiums and job polarization described in \textcite[Sections 2.4, 2.5]{Acemoglu2011skills}.
Both \condensite methods provide smooth fits, with the shape of the \condensiteNN estimate being slightly closer to the local empirical density for the 16 years of education.
Although the \flexcodeTree estimate exhibits the correct trend, it is notably more bumpy, which is presumably an artifact resulting from the cosine basis expansion. 
The other methods from the literature (shown in \cref{fig:skill_prem_lit}) also capture the correct trend, yet either yield undesirably high-variance estimates akin to \flexcodeTree, or over-smooth the more concentrated density for $12$ years of education.
\begin{figure}[H]
    \centering
    \includegraphics[width=\linewidth]{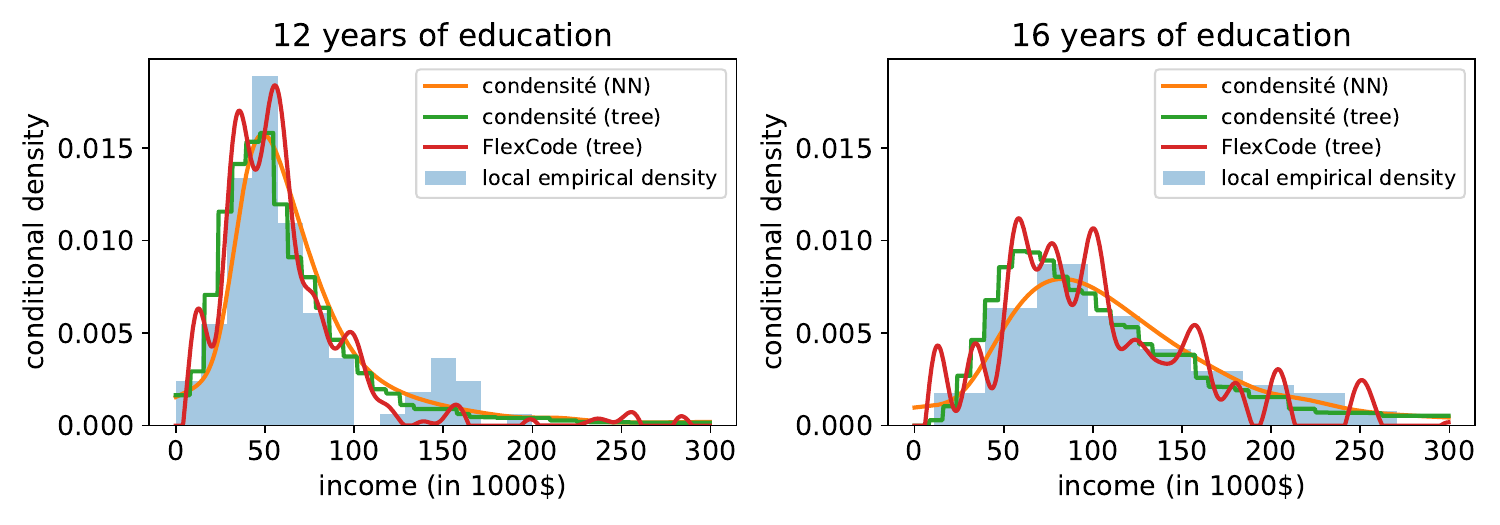}
    \vskip-1em
    \caption{Conditional income density for 12 and 16 years of education with otherwise identical covariates.}
    \label{fig:skill_prem}
\end{figure}

\paragraph{Geographic income dispersion.}
Another frequent subject of analysis in inequality research is geographic income dispersion.
In this context, different groups, e.g.\ with different levels of education are usually analyzed separately \parencite[see e.g.][]{baum2012understanding}.
Conditional density estimation allows for a fine-grained distinction of groups on the level of specific covariate sets. 
For illustration we choose a landmark with the characteristics specified in the table below.
For reference, we estimate a local empirical density for observations within $10$ years of age and within $10$ weekly work hours.
This yields $84$ observations for the non-metropolitan landmark, and $118$ observations for the metropolitan landmark
\begin{table}[H]
    \centering
    \begin{tabular}{ll}
        \textit{characteristic} & \textit{landmark}\\
        \hline
        personal & $40$ years, female, two children, white (race)\\
        \textbf{geography} & \textbf{non-metropolitan} $\boldsymbol{|}$ \textbf{metropolitan ($\boldsymbol{5+}$ million)}\\
        nativity & US-born (self and parents)\\
        work hours & $40$ (weekly, constant), private sector wage or salary\\
        education & 16 years
    \end{tabular}
    \label{tab:lm_geographic}
\end{table}
\noindent
The results are shown in \cref{fig:lm_metro}.
In the methods' estimates, as well as the local empirical density, we observe higher wages and a greater wage dispersion for the metropolitan covariate set compared to the non-metropolitan one. 
This is in line with \cite{baum2012understanding}, who report strong urban wage premiums for highly skilled workers such as those specified by our covariates.
The \flexcodeTree estimate is closer to the local empirical non-metropolitan density than the \condensite methods, yet again suffers from pronounced spurious bumps.
The other methods from the literature (see \cref{fig:lm_metro_lit}) capture the same trend as the ones presented here but either yield high-variance estimates akin to \flexcodeTree, or over-smooth the more concentrated non-metropolitan density.
\begin{figure}[H]
    \centering
    \includegraphics[width=\linewidth]{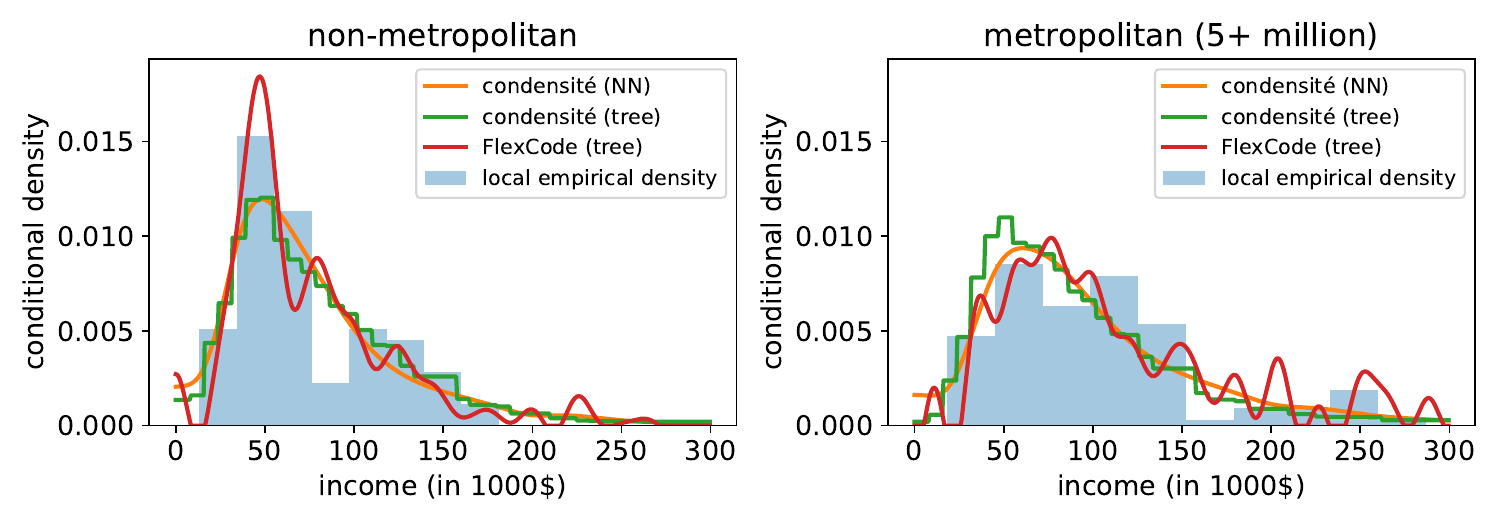}
    \vskip-1em
    \caption{Conditional income density for (non-)metropolitan inhabitants with otherwise identical covariates.}
    \label{fig:lm_metro}
\end{figure}

\subsection{ESA ICC Satellite Imaging Data}\label{sec:AGB}

For an evaluation and comparison of our methods on unstructured data in a realistic setting, we compile a dataset on above-ground biomass (AGB) estimation.
We use AGB labels for the year 2020 provided by the European space agency's (ESA) climate change initiative (CCI)\footnote{https://archive.ceda.ac.uk/} as described in \cite{esacci2025AGB}, and corresponding Sentinel-1A satellite images\footnote{https://geodes-portal.cnes.fr} as features.
AGB is an ``essential climate variable'' \parencite{penman2003good}, and as such of crucial interest to climate modelling.
Uncertainty quantification in AGB modelling is listed as a desirable target in \textcite[][Table 2-1]{esacci2025AGB}, making this an application of conditional density estimation with high potential impact \parencite[see also][]{araza2022comprehensive}.
We construct a dataset by taking the average AGB in tons per square kilometer, and corresponding $100\times100$ two-channel satellite image patches.
We choose the geographic region, weather, and season to ensure the sensor data used for imaging is informative about AGB.
Our training and test dataset are constructed from two disjoint regions of tropical savannah in the Northern Territory of Australia, and contain 14653 and 14683 observations respectively. 
We train and evaluate all algorithms on the training region (chosen to be the one with the greater range of AGB values) using 80\% of the data for training, and the remaining 20\% as a hold-out set for final performance evaluation. Then, we additionally evaluate the methods on the test region.
We refer to \cref{app:ESA_pp} for more details on the data, a description of the preprocessing steps, and a visualization of the satellite images from which the features patches are extracted.

\subsubsection{Empirical Comparison.}
For this application on AGB satellite image data, we use \condensiteCNN, a version using a convolutional neural network (CNN) and a fully connected neural network head with skip connections between them.
To take full advantage of the end-to-end training enabled by \condensite, we provide the auxiliary target coordinates $Y'_{im}$ as a third image channel to the convolutional layers, and again as a separate feature to the head.
We use the same architecture, save for the auxiliary sample coordinates, as basis coefficient estimator in \flexcodeCNN.
To conduct as fair a comparison as possible, we use a neural network of the same architecture trained to predict the labels as a feature extractor for the other methods. The features are the CNN activations of the network.
For further details and hyperparameters, see \cref{app:hyper}.

\cref{tab:res_agb_train} shows the methods' performance on the held-out 20\% of the training region dataset, and \cref{tab:res_agb_test} shows the performance on the complete test region dataset.
Performance is measured in terms of ISE (lower is better).
We find that the \condensite methods are among the best-performing methods in both evaluations, with the two CNN-based methods \condensiteCNN and \flexcodeCNN performing best of all.
We observe further that the ISE results are consistently worse on the held-out 20\% of the training region data than on the test region data. This can be explained by the training region being closer to the coast and to the equator and thus having greater biomass. As a result, it presents a more complex estimation task with fewer of the near-zero AGB areas that characterize the test data region (compare the label panel of \cref{fig:panel_test} and \cref{fig:panel_train}).
Although this complicates the quantitative interpretation of the test region performance, we note that the \condensite methods and \flexcodeCNN are the top-performing methods on both datasets. The performance difference between the training and test data is particularly pronounced for \condensier; we treat this point in further detail in the following qualitative assessment.

\noindent
{
\hfill
\begin{minipage}{0.45\textwidth}
    \begin{table}[H]
        \centering
        \begin{tabular}{lc}
\toprule
 & \textbf{ESA ICC AGB} \\
\midrule
\textit{condensité (CNN)} & $-0.9224$ \\
\textit{condensité (tree)} & $-0.8785$ \\
\textit{FlexCode (CNN)} & $-0.8929$ \\
\textit{LinCDE} & $-0.8329$ \\
\textit{DRF} & $-0.7775$ \\
\textit{condensier} & $-0.5565$ \\
\bottomrule
\end{tabular}

        \vskip-1em
        \caption{\textbf{Training region.} ISE on held-out 20\% of data.}
        \label{tab:res_agb_train}
    \end{table}
\end{minipage}%
\hskip2.2em
\begin{minipage}{0.45\textwidth}
    \begin{table}[H]
        \centering
        \begin{tabular}{lc}
\toprule
 & \textbf{ESA ICC AGB} \\
\midrule
\textit{condensité (CNN)} & $-2.1245$ \\
\textit{condensité (tree)} & $-1.9646$ \\
\textit{FlexCode (CNN)} & $-2.2726$ \\
\textit{LinCDE} & $-1.5222$ \\
\textit{DRF} & $-1.3861$ \\
\textit{condensier} & $-1.8709$ \\
\bottomrule
\end{tabular}

        \vskip-1em
        \caption{\textbf{Test region.} ISE on all data points. $\phantom{lalala}$}
        \label{tab:res_agb_test}
    \end{table}
\end{minipage}
\hfill
}

\subsubsection{Visualization and Qualitative Assessment}
The ISE results in \cref{tab:res_agb_train} and \cref{tab:res_agb_test} give a rough idea about comparative model performance, but other factors such as smoothness or variance of the estimates may be relevant for downstream tasks.
To obtain a qualitative understanding of the methods' performances, we visualize the tail width, skew, and mode of the per-pixel estimate for each method.
We measure tail width as the difference between the ninth decile $d_9$ and first decile $d_1$.
To measure skewness, we compute the Bowely-type \parencite{bowley1920} statistic 
\begin{align}
    \frac{(d_9-m)-(m-d_1)}{d_9-d_1},
\end{align}
where $m$ is the median; we upscale the value to the AGB range for greater visual contrast.
For comparison, we also show the AGB labels.
In \cref{fig:panel_test} below we show visualizations of the test region data estimates for all methods.
A visualization of all methods on the complete training region dataset can be found in \cref{app:viz_AGB}.

\begin{figure}[H]
    \centering
    \begin{subfigure}{\linewidth}
        \centering
        \includegraphics[width=\linewidth,trim={0 1.4cm 0 .95cm},clip]{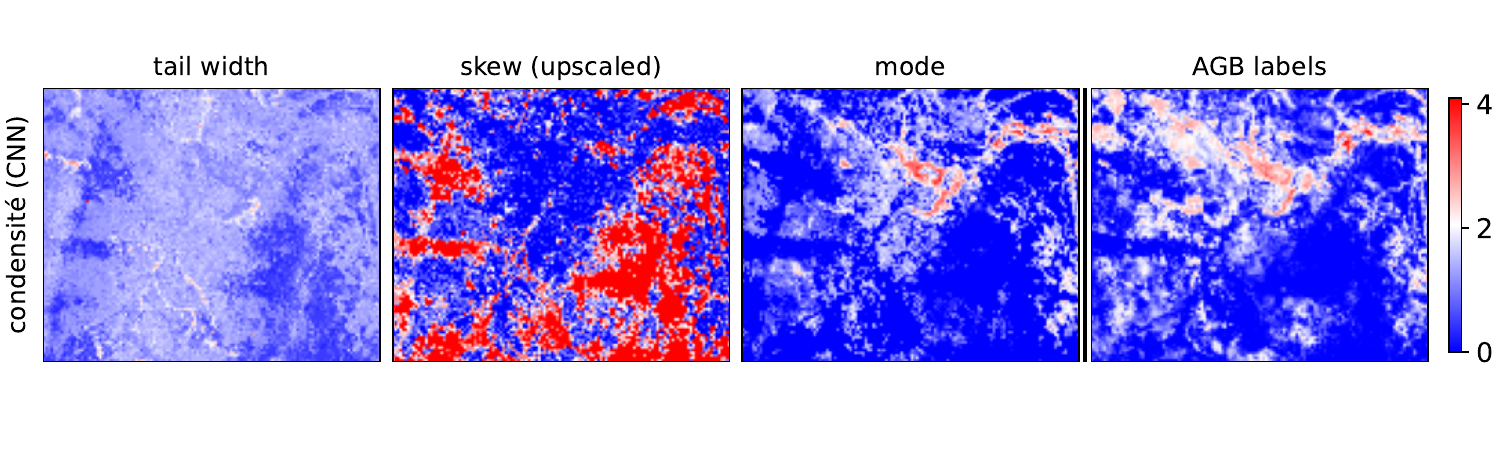}
    \end{subfigure}
        \begin{subfigure}{\linewidth}
        \centering
        \includegraphics[width=\linewidth,trim={0 1.4cm 0 1.4cm},clip]{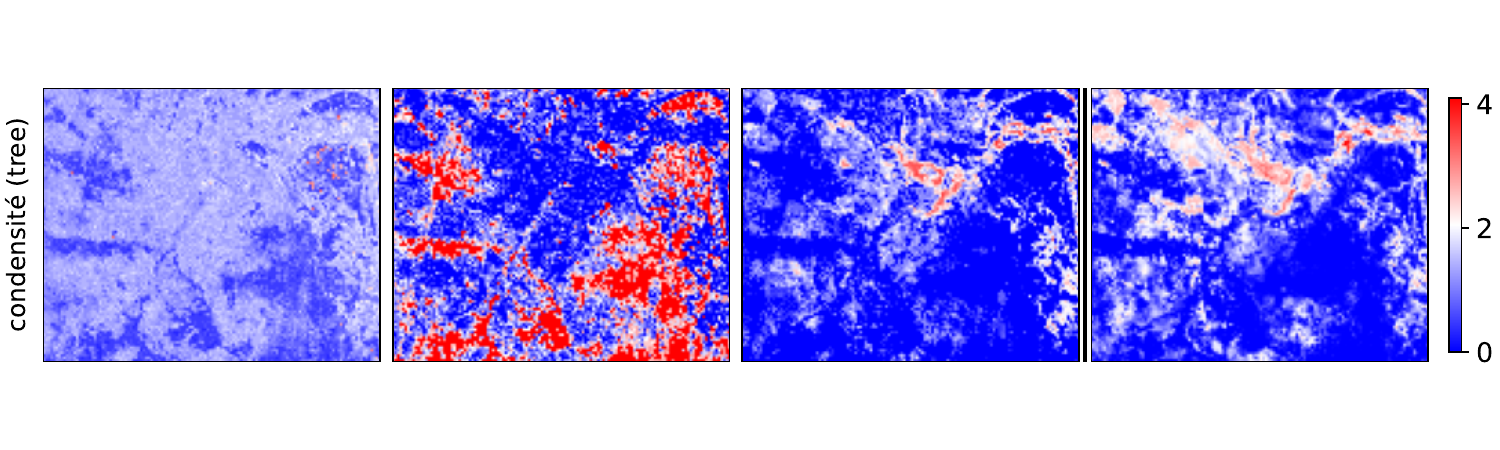}
    \end{subfigure}
    \begin{subfigure}{\linewidth}
        \centering
        \includegraphics[width=\linewidth,trim={0 1.4cm 0 1.4cm},clip]{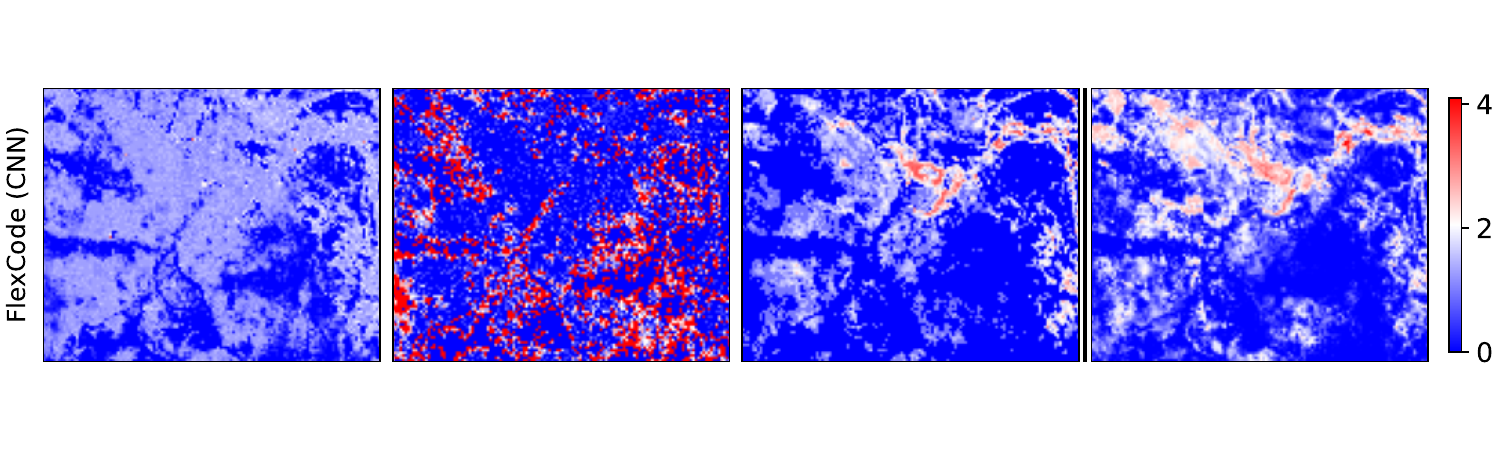}
    \end{subfigure}
    \begin{subfigure}{\linewidth}
        \centering
        \includegraphics[width=\linewidth,trim={0 1.4cm 0 1.4cm},clip]{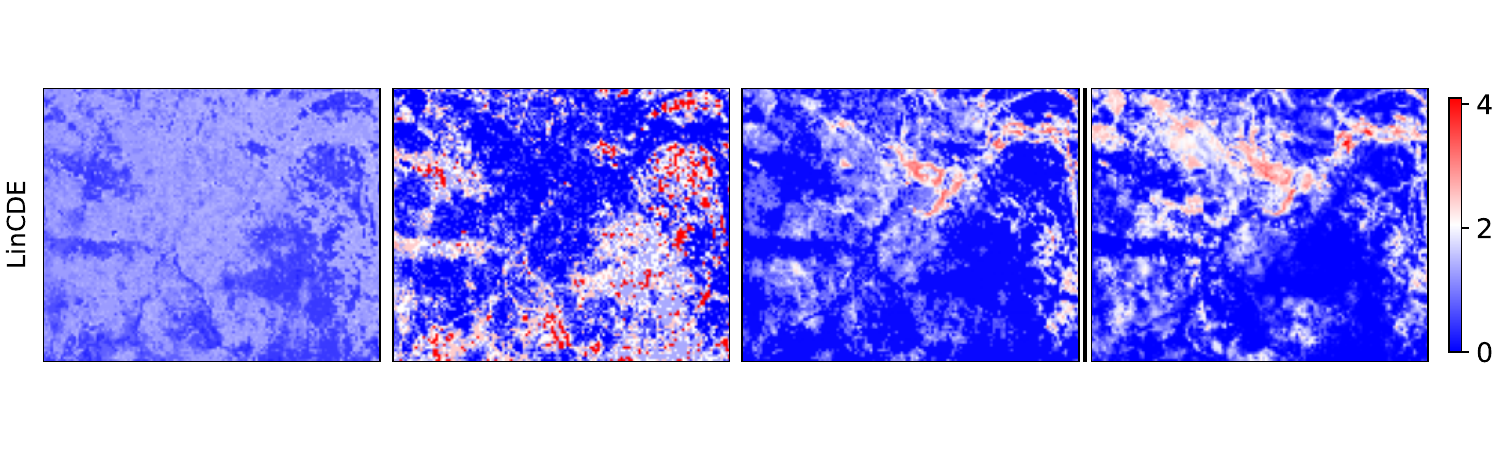}
    \end{subfigure}
    \begin{subfigure}{\linewidth}
        \centering
        \includegraphics[width=\linewidth,trim={0 1.4cm 0 1.4cm},clip]{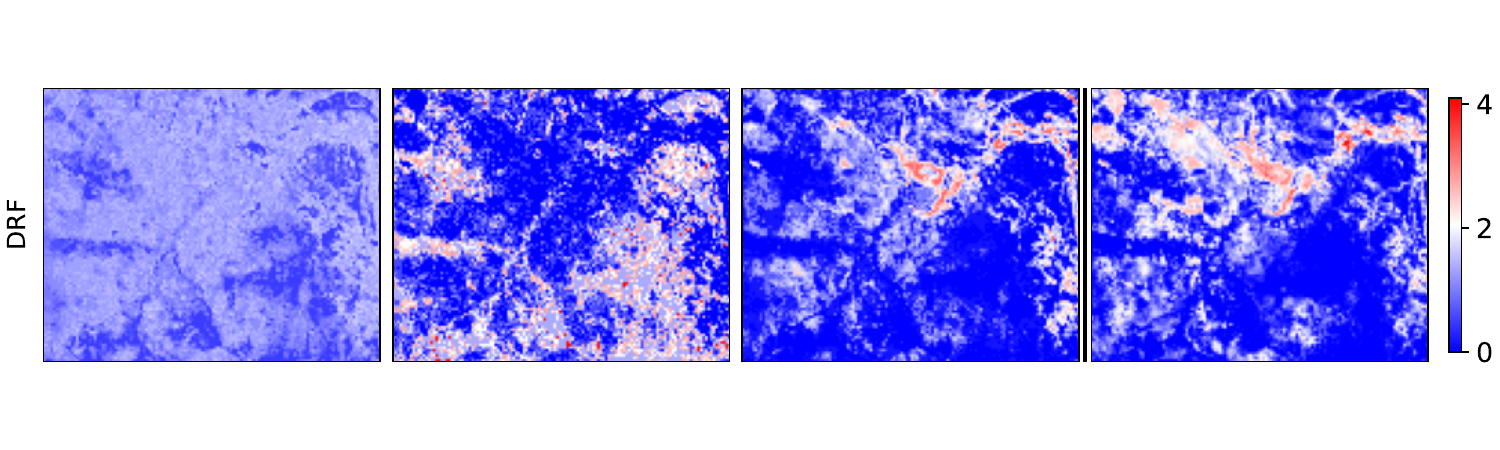}
    \end{subfigure}
    \begin{subfigure}{\linewidth}
        \centering
        \includegraphics[width=\linewidth,trim={0 1.4cm 0 1.4cm},clip]{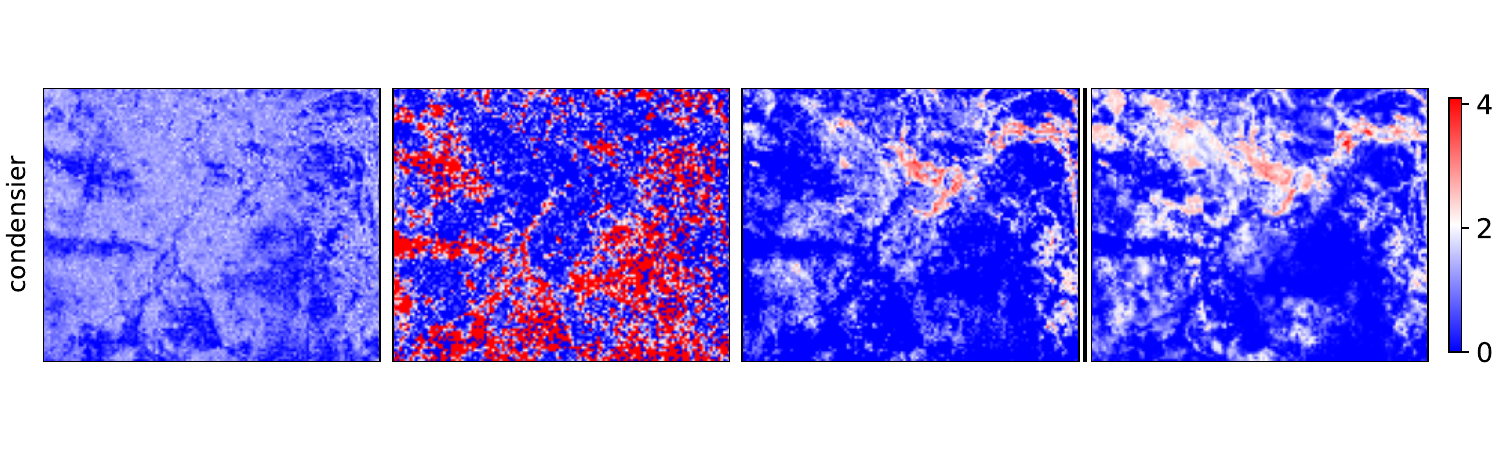}
    \end{subfigure}
    \caption{\textbf{Test region.} Visualization of method estimate summary statistics and labels.}
    \label{fig:panel_test}
\end{figure}
\noindent
We find that the mode closely resembles the labels for all methods, indicating a stable dominant peak in the conditional density.
Tail width is lowest in low-AGB areas, indicating narrow densities.
The test region has many such areas, making the density estimation task comparatively less complex than for the training region (cf. \cref{fig:panel_train}).
The two best-performing methods \condensiteCNN and \flexcodeCNN appear to have the greatest contrast in tail width between high- and low-AGB areas, although the latter generally has narrower tails. 
\textcite[][Section 4]{araza2022comprehensive} assert a positive association between AGB estimates and uncertainty across AGB datasets, hence we may expect such a contrast.
The two worst-performing methods on the test region, \drf and \lincde, have more uniform tail widths, indicating possible underfitting.
We observe the greatest difference between the model estimates is in their skew.
Both \condensite methods have high skewness in contiguous areas corresponding to low-AGB areas.
Skewness is similarly high though more dispersed for \condensier.
For \flexcodeCNN, skewness appears to be highest at the borders of low-AGB areas.
To a lesser extent, this can also be seen for \lincde and \drf, which are somewhere in between the \condensite methods and \flexcodeCNN, although again with less pronounced contrasts between high- and low-AGB areas.
Overall, the experiment demonstrates that the \condensite methods, like \flexcode, can match or outperform the other methods in the literature in terms of ISE. 
It also highlights the complementary nature of the qualitative aspects of the different methods, pointing to different strengths for different purposes.

\section{Discussion and Limitations}\label{sec:discussion}

\paragraph{Considerations for applying \condensite.}
Applying \condensite involves choosing a predictor as well as the hyperparameters $M$ and $h$.
The choice of predictor can be a major determinant for overall performance, and we recommend testing different methods and architectures.
Choosing suitable values for $M$ and $h$ is essential for good performance, but doing so is not difficult.
Higher values of $M$ and lower values of $h$ are better, and low values of $h$ require high values of $M$.
For optimal performance, it is advisable to evaluate different predictor and hyperparameter combinations, yet our results suggest that good performance can be achieved without extensive tuning.

\paragraph{Theoretical result.}
Our main theoretical result establishes the convergence of our estimator to the true conditional density in the data limit and as $h$ goes to zero.
The latter condition is inherited from the transformation into a regression task \parencite[see also][Equation 2.1]{fan1996estimation}.
There is a bias-variance trade-off between choosing $h$ large for a tighter bound by \cref{thm:1}, and choosing $h$ small for the transformation to remain close to the true density by \cref{eq:transformation_h_limit}.
Although we find clear empirical evidence corroborating the benefit of the auxiliary samples, our proof strategy does not capture an effect of $M$.
Since the auxiliary samples are only block-wise independent, the influence of $M$ cancels out in the symmetrization step.
A proof strategy that overcomes this limitation and captures the influence of $M$ would help outline the advantage gained by our approach and could guide the choice of $M$ and $h$.
We are not aware of any technique allowing this, and such a result would represent a substantial advancement.

\paragraph{Empirical evaluation.}
On synthetic data we find that increasing the number of auxiliary samples $M$ greatly and consistently improves the performance of our estimators.
This highlights the benefit of our approach compared to using the observations alone.
In our evaluation on challenging real-world datasets, we find that the \condensite estimators match or outperform the other methods in the literature.
On a large population survey dataset we find that \condensiteNN outperforms the other methods and yields estimates with desirable smoothness properties.
The latter point may be particularly interesting for applications that rely on comparing densities at different points, and underscores the ability to choose an inductive bias via the regressor as an advantage of \condensite.
On a satellite image dataset, we showcase its ability to integrate different machine learning methods and find that \condensiteCNN achieves state of the art performance.
A qualitative assessment of the estimates highlights that methods with similar performance may provide estimates with very different characteristics. 
Hence, our approach complements the existing ones in the literature beyond the raw performance of the \condensite methods.

\paragraph{Limitations.}
We do not investigate or compare the computational requirements of our method in this work, but there is no doubt that the computation time will tend to increase in the number of auxiliary samples $M$. 
Though we consider it unlikely for this to be prohibitive since \condensite can benefit from the widely available infrastructure for large-scale regression, it may slow down prototyping and complicate parameter tuning.
Unlike some other methods in the literature, \condensite also does not yet allow for multivariate dependent variables. 
Such an extension would present a natural and exciting opportunity for future research.
Our experimental setup is aimed at mimicking realistic use-cases, but future research is needed to assess how these empirical results generalize to other datasets and data types, including multimodal data.
In addition, our evaluation criteria are of a generic nature. For practitioners, domain-specific criteria or evaluations on downstream tasks may help better outline the benefit of our methods in practice.
\section{Conclusion}\label{sec:conclusion}
We propose a way of transforming conditional density estimation into a single nonparametric regression task using auxiliary samples that encourage nearby points in the feature space to yield similar estimates.
This acts as regularization and mitigates the overfitting common to flexible conditional density estimators.
We develop \condensite, a method that implements this approach.
By interpreting regression outputs as density estimates directly, rather than viewing them as an input to another parametrization, it is able to leverage the full expressive power of high-dimensional regression methods and is trainable without modification using the same libraries and infrastructure.
We establish convergence of the estimator for function classes of limited capacity and numerically study how the number of auxiliary samples affects performance. Our results show that more auxiliary samples generally improve performance, and that the hyperparameters are straightforward to tune.
In our empirical evaluation, we find that \condensite with a neural network as predictor outperforms the state of the art in the literature on a large real-world population survey dataset and matches it on a satellite imaging dataset.
Qualitatively, we find that the \condensite models provide different and complementary estimates to existing methods.
Moreover, the ability to adapt the inductive bias via the choice of regressor may make our approach interesting to a wide range of applications.
Conditional density estimation in high dimensions is often considered a challenging problem out of reach of current methods.
Our findings demonstrate that it can be feasible and effective under realistic conditions.
Overall, our results indicate that \condensite in particular, and state of the art methods more generally, hold strong promise for applied research in domains requiring flexible conditional density estimates.
\subsection*{Acknowledgements}
We thank Julia M. Schmidt for suggesting the population survey application, Rémy Abergel for his advice on satellite image datasets, and Myrto Limnios for feedback on the theoretical part of the work.
AGR received funding from the European Union's Horizon 2020 research and innovation programme under the Marie Skłodowska-Curie grant agreement No 945332 \raisebox{-.05em}{\includegraphics[height=.85em]{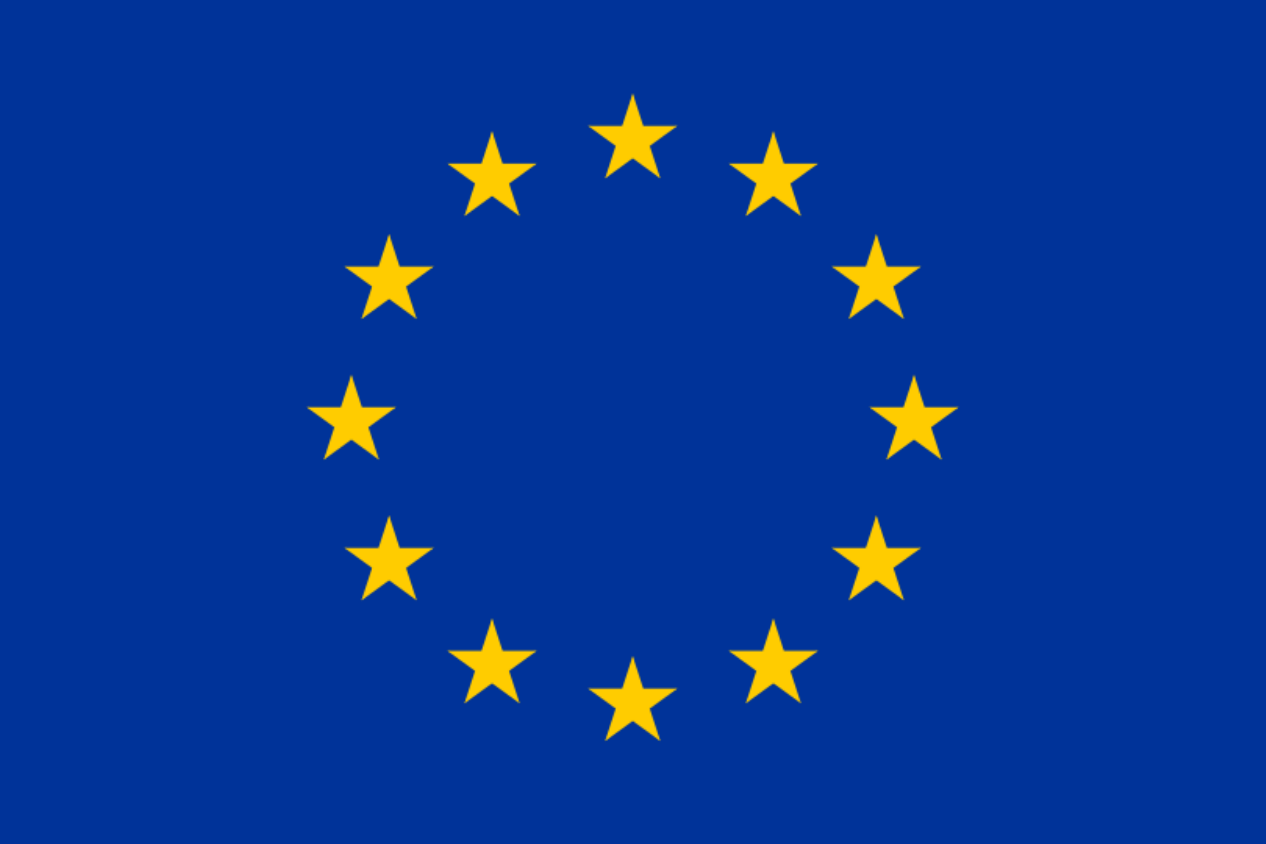}}.

\printbibliography

\newpage
\appendix
\crefalias{section}{appendix}      %
\crefalias{subsection}{appendix}   %
\crefname{appendix}{Appendix}{Appendices}
\Crefname{appendix}{Appendix}{Appendices}
\section{Glossary}\label{app:glossary}

{\renewcommand{\arraystretch}{1.5}
\begin{longtable}{|p{1.5cm}|p{13.5cm}|}
        \hline
        $(X, Y)$ & Random variable on $\X\times[0,1]$ drawn from $\P$.\\
        $\bar{Y}'$ & $\bar{Y}'=(Y'_1,\ldots,Y'_M)$ is drawn from $\left(\Unif[0,1]\right)^{\otimes M}$ independently of $(X,Y)$.\\
        $\bar{Z}_h$ & $\bar{Z}_{h}=(Z_{h,1},\ldots,Z_{h,M})$ with $Z_{h,m}=K_h(Y-Y'_m)$ for each $1\leq m\leq M$.\\
        \hline
    \caption{Glossary of random variables related to $(X,Y)$.}
    \label{tab:var_glossary}
\end{longtable}}

{\renewcommand{\arraystretch}{1.5}
\begin{longtable}{|p{1.5cm}|p{13.5cm}|}
        \hline
        $f_{X,Y}(\cdot,\cdot)$ & Joint density of $X$ and $Y$.\\
        $\fstar(\cdot|\cdot)$ & Conditional density of $Y$ given $X$.\\
        $\transformation(\cdot|\cdot)$ & 
        Conditional density of $Y_m'$ given $X$.
        \\ 
        $\projection(\cdot|\cdot)$ & Projection of $\transformation$ onto $\F$.\\
        $\estimator(\cdot|\cdot)$ & Approximation of $\projection$ on $n$ samples of $(X,Y)$.\\
        \hline
    \caption{Glossary of densities related to $(X,Y)$.}
    \label{tab:f_glossary}
\end{longtable}}

{\renewcommand{\arraystretch}{1.5}
\begin{longtable}{|p{1.5cm}|p{13.5cm}|}
        \hline
        $\X\times[0,1]$ & Sample space of $(X, Y)$.\\
        \hline
        $\P$ & Law of the experiment of interest on $\X\times[0,1]$.\\
        $\PX$ & Marginal law of $X$ under $\P$.\\
        $\Ph$ & Joint law of $(X,\bar{Y}',\bar{Z}_h)$.
        \\
        $\Phn$ & The empirical counterpart to $\Ph$ that puts mass $1/n$ on every $(X_i,\bar{Y}_i^\prime,\bar{Z}_{hi})$, where $(X_1,\bar{Y}_1^\prime,\bar{Z}_{h1})$, $\ldots$, $(X_n,\bar{Y}_n^\prime,\bar{Z}_{hn})$ are independently drawn from $\Ph$.\\
        $\Q$ & Product of $\PX$ with the uniform law on $[0,1]$.\\
        $\QnM$ & The empirical law that puts mass $1/(nM)$ on every $(X_i,Y_{im}')$.\\
        \hline
    \caption{Glossary of laws related to $(X,Y)$.}
    \label{tab:p_glossary}
\end{longtable}}

\section{Proof of the Main Result}\label{app:proof_details}

We aim to bound the probability of the excess risk $\hat\delta\coloneq\ERh(\estimator)$ of our estimator, as defined in \cref{eq:ER_loss}, exceeding a certain threshold.
In a first step (\cref{excess_risk}), we start with the insight that $\hat\delta$ can be expressed as the supremum over the empirical process $(\Phn-\Ph)$ evaluated over the loss differences of pairs of functions from a
class of functions $\F(\delta)$ with excess risk no greater than $\delta$.
The deviation of this expression from its expectation, denoted as $\phideltahat$, can be bounded probabilistically by a version of Talagrand's concentration inequality.
The probabilistic bound is only informative if it bounds an excess risk that is lower than that of at least one function.
We define $\sigma_n^t$ as the lowest excess risk beyond which that is always the case.
This allows us to use the concentration inequality in \textcite[Theorem 4.3]{koltchinskii2011oracle} to bound the chance of $\hat\delta$ exceeding $\sigma_n^t$.
In a second step (\cref{weaker_simpler}), we aim to isolate the part of $\sigma_n^t$ that depends on the unknown true law.
We denote that part as $\phisharp$ and refer to the remainder as $s_n^t$.
In a third step (\cref{symmetrization}), we aim to express the difference $(\Phn-\Ph)$ that features in $\phisharp$ solely in terms of the empirical law using Rademacher sums via a standard symmetrization argument and the use of \textcite[Corollary 1]{maurer2016vector}.
In a fourth step (\cref{complexity}), we use the assumption on the covering number integral to bound the expectation of the Rademacher sum.
Finally (\cref{combining}), we combine our bound on $\phisharp$ with a solution for the remainder $s_n^t$ and plug them in for $\sigma_n^t$ to obtain our result.

$\phantom{a}$\\[-1em]

\subsection{Controlling the excess risk}\label{excess_risk}
Our goal is to obtain a tail inequality on the excess risk 
$\hat\delta$ with respect to $\projection$.
To this end we apply \cite[][Theorem 4.3]{koltchinskii2011oracle}.

\paragraph{Defining an excess risk measure.}
For any $\delta>0$, let the $\delta$-minimal subset of functions be defined as
\begin{align*}
    \F(\delta)\coloneq \{f\in\F\colon \ERh(f)\leq\delta\}.
\end{align*}
In light of \cref{eq:ER_loss}, for any $\epsilon\in(0,\hat\delta]$, for any $f'\in\F(\epsilon)$,
\begin{align}
    \hat\delta &= \Ph(\ell[\estimator] - \ell[f']) + \Ph(\ell[f'] - \ell[\projection])\nonumber\\ 
    &\leq \Ph(\ell[\estimator] - \ell[f'])+\epsilon \nonumber\\
    &= \Phn(\ell[\estimator]-\ell[f'])+(\Ph-\Phn)(\ell[\estimator]-\ell[f'])+\epsilon\nonumber\\
    &\leq (\Ph-\Phn)(\ell[\estimator]-\ell[f'])+\epsilon\nonumber\\
    &\leq \sup_{f_1,f_2\in\F(\hat\delta)}|(\Phn-\Ph)(\ell[f_1]-\ell[f_2])|+\epsilon.
    \label{eq:proof_strat}
\end{align}
This relates the excess risk $\hat\delta$ to the supremum of the empirical process $(\Phn-\Ph)$ over
\begin{align*}
    \Lambda(\hat\delta) \coloneq \{\ell[f_1]-\ell[f_2]\colon f_1,f_2\in\F(\hat\delta)\},
\end{align*}
which we denote as 
\begin{align}
    \|\bar{P}_{h,n} - \Ph\|_{\Lambda(\hat\delta)} \coloneq \underset{\lambda\in\Lambda(\hat\delta)}{\sup}|(\bar{P}_{h,n}-\Ph)\lambda|
    \label{eq:loss_difference_approx_error}.
\end{align}
The deviation of \cref{eq:loss_difference_approx_error} from its expectation 
\begin{align}
    \phideltahat \coloneq \mathbb{E}_{\Ph}\|\bar{P}_{h,n}-\Ph\|_{\Lambda(\hat\delta)},
    \label{eq:phideltahat}
\end{align}
can be bounded using the \cite{bousquet2002bennett} version of Talagrand's inequality \parencite[][]{talagrand1996new}.\\

\paragraph{A distribution-dependent upper bound.}
For said inequality, we introduce the squared diameter of $\F(\delta)$ as
\begin{align*}
    D_h(\delta) \coloneq \underset{f_1,f_2\in\F(\delta)}{\sup}\Ph(\ell[f_1]-\ell[f_2])^2.
\end{align*}
Now, fix arbitrarily $t>0$.
For any $\delta>0$ one can define a probabilistic upper bound $U_n^t(\delta)$ on \cref{eq:loss_difference_approx_error} such that, if $\delta\geq \sup\{\delta\in(0,1]\colon \delta\leq U_n^t(\delta)\}$, then the bound is exceeded with probability at most $e^{1-t}$.
Let the $\flat$-transform for any non-negatively valued function $\delta\mapsto\psi(\delta)$ be
\begin{align*}
    \delta\mapsto\psi^\flat(\delta)\coloneq\sup_{\sigma\geq\delta}\frac{\psi(\sigma)}{\sigma}.
\end{align*}
For any $\delta>0$, let $\Lambda(\delta)$ and $\phi_{h,n}(\delta)$ be defined as in \cref{eq:loss_difference_approx_error} and \cref{eq:phideltahat}, with $\delta$ substituted for $\hat\delta$, and then let
\begin{align}
    V_n^t(\delta) \coloneq 4\left[\phi^\flat_{h,n}(\delta) + \sqrt{D_h^\flat(\delta)}\sqrt{\frac{t}{n\delta}}+\frac{t}{n\delta}\right]
    \label{eq:Vnt}
\end{align}
be an upper bound on $(U_n^t)^\flat(\delta)$.
Note that $D_h^\flat(\delta)$ can be upper bounded by a constant (see the forthcoming \cref{prop:diameter_bound}), and thus every term of $V_n^t(\delta)$ is decreasing in $\delta$.
Let
\begin{align}
    \sigma_n^t\coloneq \inf\{\delta>0\colon V_n^t(\delta)\leq 1\}.
    \label{eq:sigma}
\end{align}
In essence, $V_n^t(\sigma_n^t)\leq 1$, hence $U_n^t(\sigma_n^t)\leq \sigma_n^t$, which implies that $\sigma_n^t\geq\sup\{\delta\in(0,1]\colon\delta\leq U_n^t(\delta)\}$.
Therefore, $\sigma_n^t$ is a probabilistic upper bound on \cref{eq:loss_difference_approx_error}, which, in view of \cref{eq:proof_strat}, is itself an upper bound on $\ERh(\estimator)$.
In summary, \cite[][Theorem 4.3]{koltchinskii2011oracle} yields the distribution-dependent concentration inequality
\begin{align}
    \Ph\left(\ER_h(\estimator)\geq\sigma_n^t\right) \leq e^{1-t}.
    \label{eq:proof_part1}
\end{align}

\subsection{Deriving a weaker but simpler control of the excess risk}\label{weaker_simpler}

In this step we aim to isolate the part of the risk bound from \cref{eq:proof_part1} that depends on the unknown true distribution $\Ph$.
The definition of $\sigma_n^t$ in \cref{eq:sigma} can be separated into two parts.
Let the first part be given by 
\begin{align}
    \phisharp(1/8) \coloneq \inf\left\{\delta>0\colon \phi^\flat_{h,n}(\delta)\leq1/8\right\}.
    \label{eq:phisharp}
\end{align}
For the second part we derive an upper bound on $D_h^\flat$ from \cref{eq:Vnt} based on the following Lemma.
\begin{lemma}
    For all $\delta>0$,
    $\F(\delta) \subseteq \{f\in\F\colon \Q(f-\projection)^2\leq\delta\}$.
    \label{lemma:D_upper_bound}
\end{lemma}
\begin{proof}
    Consider the functional
    \begin{align*}
        \psi\colon \F&\to\mathbb{R}_+,\\
        f&\mapsto \frac{1}{2}\Q(\transformation-f)^2.
    \end{align*}
    It is convex and its gradient at any $f\in\F$ is given by the Riesz representer $\nabla\psi(f) = -(\transformation-f)$.
    This means in particular that $\nabla\psi(f)\cdot g = \langle -(\transformation-f), g\rangle_{\Q}$.
    The projection of $\transformation$ onto $\F$ is the minimizer $\projection = \argmin_{f\in\F}\psi(f)$.
    Since $\F$ is closed and convex, it follows that for all $f\in\F$,
    \begin{align*}
        \langle -(\transformation-\projection), f-\projection \rangle_{\Q} \geq 0.
    \end{align*}
    This implies the inequality below, from which the result follows:
    \begin{align*}
        \Q(f-\transformation)^2 &= \Q(f-\projection+\projection-\transformation)^2\\
        &= \Q(f-\projection)^2+\Q(\projection-\transformation)^2+2\langle f-\projection,\projection-\transformation\rangle_{\Q}\\
        &\geq \Q(f-\projection)^2+\Q(\projection-\transformation)^2.
    \end{align*}
\end{proof}
\begin{proposition}[Upper bound on $D_h^\flat(\delta)$]
    For all $\delta>0$, $D_h^\flat(\delta)\leq16\frac{c^2}{h^2}.$
    \label{prop:diameter_bound}
\end{proposition}
\begin{proof}
    Note that for all $x\in\X$, $\bar{y}\in[0,1]^M$, $\bar{z}\in[0,c/h]^M$, and $f_1, f_2\in\F(\delta)$,
    \begin{align*}
        (\ell[f_1](x,\bar{y},\bar{z})-\ell[f_2](x,\bar{y},\bar{z}))^2
        &\leq \left(\frac{1}{M}\sum_{m=1}^M \left|\left(z_m-f_1(y_m|x)\right)^2 - \left(z_m-f_2(y_m|x)\right)^2\right|\right)^2\\
        &\leq \left(2\frac{c}{h}\times\frac{1}{M}\sum_{m=1}^M\left|(f_1-f_2)(y_m|x)\right|\right)^2,
    \end{align*}
    since each of the squared differences in the absolute value of the first line is bounded by $(c/h)^2$.
    Note that the squared distance between any two functions in $\F$ is at most $4\delta$ by the triangle inequality, since $\F(\delta)$ is a subset of $\{f\in\F\colon \Q(f-\projection)^2\leq\delta\}$ by \cref{lemma:D_upper_bound}.
    It follows that by convexity (second inequality) and the said squared triangle inequality in $\F(\delta)$ (third inequality),
    \begin{align*}
        \Ph(\ell[f_1]-\ell[f_2])^2\leq 4\frac{c^2}{h^2}\Q(f_1-f_2)^2
        &\leq 4\frac{c^2}{h^2}\sup_{f_1,f_2\in\F(\delta)}\Q(f_1-f_2)^2\leq 16\frac{c^2}{h^2}\delta.
    \end{align*}
    From the definitions of $D_h(\delta)$ and the $\flat$-transform, it follows directly that $D_h^\flat(\delta)\leq16\frac{c^2}{h^2}$.
\end{proof}
Given the bound on $D_h^\flat(\delta)$ from \cref{prop:diameter_bound}, we return to defining the second part of $\sigma_n^t$.
Let $s_n^t$ be the unique solution in $\delta$ to 
\begin{align*}
    \sqrt{16\frac{c^2}{h^2}}\sqrt{\frac{t}{n\delta}} + \frac{t}{n\delta} = 1/8,
\end{align*}
given by 
\begin{align}
    \frac{t}{n}\left(\sqrt{\left(\frac{2c}{h}\right)^2+\frac{1}{8}}-\frac{2c}{h}\right)^{-2}.
    \label{eq:s_t_solution}
\end{align}
Since both terms in the right-hand side expression of \cref{eq:Vnt} are decreasing in $\delta$, it straightforwardly holds that $\sigma_n^t \leq \phisharp(1/8) + s_n^t$.
Plugging this into \cref{eq:proof_part1}, we obtain
\begin{align}
    \Ph(\ER(\estimator)\geq \phisharp(1/8) + s_n^t) \leq e^{1-t}.
    \label{eq:concentration_ineq_too_complicated}
\end{align}
In the following steps we aim to express $\phisharp$ in terms of the empirical distribution using the geometry of the function class $\F$.
\noindent

\subsection{Symmetrization}\label{symmetrization}
The object $\phisharp$ still depends on the unknown distribution $\Ph$ via \cref{eq:phideltahat}.
Our goal in this step is to express the difference between empirical and theoretical distribution solely in terms of the empirical distribution using Rademacher sums. These can then be upper bounded in terms of the complexity of the class $\F$.
For any $\delta>0$, we introduce the set of differences
\begin{align*}
    \G(\delta) \coloneq \left\{f-\projection\colon f\in\F(\delta)\right\}.
\end{align*}
Let $\epsilon_{11}, \dots, \epsilon_{nM}$ be independent Rademacher random variables drawn independently of $\Phn$. 
For any $g \in \G(\delta)$, we define the Rademacher sum
\begin{align*}
    R_{h,n}g\coloneq \frac{1}{nM} \sum_{i=1}^n\sum_{m=1}^M \epsilon_{im}g(Y'_{im}|X_i),
\end{align*}
and denote its supremum analogously to \cref{eq:loss_difference_approx_error} as
\begin{align*}
    \|R_{h,n}\|_{\G(\delta)} = \underset{g\in \G(\delta)}{\sup}|R_{h,n}g|.
\end{align*}
We now relate first $\phi_{h,n}(\delta)$ and then $\phisharp(\delta)$, whose definitions involve $\F(\delta)$, to $\|\Radsum\|_{\G(\delta)}$, whose definition involves $\G(\delta)$. 
Let henceforth $\E_{\Ph,\bfeps}$ be the expectation operator with respect to the product of $(\Ph)^{\otimes n}$ and $\Rad(1/2)^{\otimes nM}$.
By the standard symmetrization argument \parencite[see e.g. the rightmost inequality of the particular case of][Theorem 2.1]{koltchinskii2011oracle}, we have that
\begin{align}
    \phi_{h,n}(\delta) \leq 2 \E_{\Ph,\bfeps}\left[\sup_{f_1,f_2\in\F(\delta)}\left|\frac{1}{n}\sum_{i=1}^n\epsilon_{i1}(\ell[f_1]-\ell[f_2])(X_i,\bar{Y}_i',\bar{Z}_i)\right|\right].
    \label{eq:first_rademacher}
\end{align}
The absolute value can be dropped. Then, using $\ell[\projection]$ as a pivot, and the fact that Rademacher law is symmetric around zero, \cref{eq:first_rademacher} is equivalent to 
\begin{align}
    \phi_{h,n}(\delta) \leq 4 \E_{\Ph,\bfeps}\left[\sup_{f\in\F(\delta)}\left(\frac{1}{n}\sum_{i=1}^n\epsilon_{i1}(\ell[f]-\ell[\projection])(X_i,\bar{Y}_i',\bar{Z}_i)\right)\right].
    \label{eq:second_rademacher}
\end{align}
For any $f\in\F$, $\ell[f](X_i,\bar{Y}_i',\bar{Z}_i)$ can be seen as the value of the averaging function $u\mapsto h(u)=M^{-1}\sum_{m=1}^Mu_m$ at $\bar{\ell}[f](X_i,\bar{Y}_i',\bar{Z}_i)$ whose $m$-th component is $(f(Y'_{im}|X_i)-Z_{im})^2$.
\cite[Corollary 1 in][]{maurer2016vector} provides a version of the contraction inequality \parencite[see e.g.][Theorem 2.2]{koltchinskii2011oracle} for this case.
The function $h$ is $M^{-1/2}$-Lipschitz, that is, $|h(u)-h(u')|\leq M^{-1/2}d_2(u,u')$ for all $u,u'\in\mathbb{R}$ and with $d_2$ the Euclidean distance on $\mathbb{R}^M$. Thus, by the corollary, \cref{eq:second_rademacher} yields
\begin{align}
    \phi_{h,n}(\delta) &\leq 4\E_{\Ph,\bfeps}\left[\sup_{f\in\F(\delta)}\left(\frac{1}{n}\sum_{i=1}^n\epsilon_{i1}h\left((\bar{\ell}[f]-\bar{\ell}[\projection])(X_i,\bar{Y}'_i,\bar{Z}_i)\right)\right)\right]\nonumber\\
    &\leq \frac{4\sqrt{2}}{\sqrt{M}}\E_{\Ph,\bfeps} \Bigg[\sup_{f\in\F(\delta)}\Bigg(\frac{1}{n}\sum_{i=1}^n\sum_{m=1}^M\epsilon_{im}\Big[(f(Y'_{im}|X_i)-Z_{im})^2-(\projection(Y'_{im}|X_i)-Z_{im})^2\Big]\Bigg)\Bigg]
    \label{eq:by_corollary}.
\end{align}
Note that the $M^{-1/2}$ term here is why our final result does not capture an influence of $M$.
The expression in square brackets in \cref{eq:by_corollary} is of the form $(a-x)^2-(b-x)^2$
and can be expressed of in terms of the function $\phi_{b,x}\colon t\mapsto(t+b-x)^2-(b-x)^2$ over $[-c/h,c/h]$ if one chooses $t=a-b$.
For any $b, x\in[0,c/h]$ the function $\phi_{b,x}$ satisfies $\phi_{b,x}(0)=0$ and
\begin{align*}
    |\phi_{b,x}(u)-\phi_{b,x}(v)| = |(u-v)[(u-(x-b)) + (v-(x-b))]|
    \leq L |u-v|
\end{align*}
with $L=4c/h$, since $u,v\in[-c/h,c/h]$.
It follows that the function $u\mapsto \phi_{b,x}(u)h/(4c)$ is a contraction.
Thus, by the aforementioned contraction inequality, \cref{eq:by_corollary} is itself smaller than 
\begin{align*}
    \frac{32\sqrt{2}}{\sqrt{M}}\frac{c}{h}\E_{\Ph,\bfeps}\left[\sup_{f\in\F(\delta)}\left(\frac{1}{n}\sum_{i=1}^n\sum_{m=1}^M\epsilon_{im}\left[f(Y_{im}'|X_i)-\projection(Y_{im}'|X_i)\right]\right)\right]
    =\frac{32c\sqrt{2M}}{h}\E_{\Ph,\bfeps}{\|\Radsum\|}_{\G(\delta)}.
\end{align*}
It follows that
\begin{align}
    \phiflat(\delta)\leq \frac{32c\sqrt{2M}}{h}\sup_{\sigma\geq\delta}\frac{\E_{\Ph,\bfeps}\|\Radsum\|_{\G(\sigma)}}{\sigma}.
    \label{eq:phiflat}
\end{align}
This puts us in a position where we can upper bound $\phisharp$ by upper bounding the Rademacher sum for a given complexity of $\F$.
\subsection{Bounding the excess risk with respect to the complexity of \texorpdfstring{$\F$}{\textit{F}}.}\label{complexity}
On the basis of the covering number assumption from \cref{eq:covering_number_condition}, we adapt \textcite[Theorem~3.12]{koltchinskii2011oracle} to our case in order to obtain a simpler bound on the excess risk.

In the previous \cref{symmetrization} we rely on the laws $\Ph$ and $\Phn$ on $\X\times[0,1]^M\times\mathbb{R}_+^M$.
The functions in $\G(\delta)$, like those in $\F(\delta)$ are defined on $\X\times[0,1]$.
We therefore introduce $\QnM$, the empirical law that puts mass $1/(nM)$ on each $(X_i,Y_{im}')$.
The law $\QnM$ is a marginal law over $\X\times[0,1]$ derived from $\Phn$, just like how $\Q$ is derived from $\Ph$.
In view of \cref{eq:covering_number_condition} we then capture the richness \parencite[see][Section 3.4]{koltchinskii2011oracle} of the class $\G(\delta)$ as
\begin{align}
    \sigma_n^2 \coloneq \sup_{g\in\G(\delta)}\QnM g^2.
    \label{eq:sigma_n}
\end{align}
The covering number condition in \cref{eq:covering_number_condition} also holds for $\G(\delta)$, which is also uniformly absolutely bounded by $c/h$ for any $\delta$. 
We then have by the version in \textcite[][Theorem 3.11]{koltchinskii2011oracle} of Dudley's entropy integral \parencite[see][]{dudley2014uniform} that
\begin{align}
    \E_{\Ph,\bfeps}\|R_{h,n}\|_{\G(\delta)} &\lesssim \frac{1}{\sqrt{nM}}\E_{\Ph}\int_0^{2\sigma_n}\sqrt{\log N(\varepsilon,\F,L^2(\Phn))}\dd\varepsilon \nonumber\\
    &= \frac{1}{\sqrt{nM}}\E_{\Ph}\int_0^{2\sigma_n}\sqrt{\log N(\varepsilon,\F,L^2(\QnM))}\dd\varepsilon \nonumber\\ 
    &\lesssim \frac{1}{\sqrt{nM}} \E_{\Ph}\int_0^{2\sigma_n} \left(\frac{\NFQn}{\varepsilon}\right)^\rho \dd\varepsilon,
    \label{eq:covering_bound}
\end{align}
where the last step uses \cref{eq:covering_number_condition}.
We aim to upper bound the integral in \cref{eq:covering_bound}.
By change of variable with $u\coloneq\varepsilon/\NFQn$,
\begin{align*}
    \int_0^{2\sigma_n}\left(\frac{\NFQn}{\varepsilon}\right)^\rho \dd\varepsilon = \NFQn\int_0^{2\sigma_n/\|F_h\|_{L^2(\QnM)}}u^{-\rho}\dd u.
\end{align*}
Importantly we have that for all $f\in\F$ and any law $Q$ on $\X\times[0,1]$, it holds that $Qf^2 \leq QF^2$, hence we can upper bound the previous equality by
\begin{align}
    \NFQn\int_0^2 u^{-\rho}\dd u =  \NFQn\frac{2^{1-\rho}}{1-\rho}.
    \label{eq:rad_upp_bound}
\end{align}
Without loss of generality, we can assume that $0\leq F_h\leq c/h$. Therefore,
\begin{align*}
    \E_{\Ph,\bfeps}\|R_{h,n}\|_{\G(\delta)} \lesssim \frac{1}{\sqrt{nM}}\frac{c}{h}\frac{2^{1-\rho}}{1-\rho}.
\end{align*}

\subsection{Combining the Parts}\label{combining}

Making use of \cref{eq:rad_upp_bound} we can upper bound \cref{eq:phiflat} and write
\begin{align*}
    \phiflat(\delta) \lesssim
    \frac{c\sqrt{M}}{h\delta}\frac{1}{\sqrt{nM}}\frac{c}{h}\frac{2^{1-\rho}}{1-\rho}
    = \frac{c^2}{h^2\delta\sqrt{n}}\frac{2^{1-\rho}}{1-\rho}.
\end{align*}
Note that here one can see how the effect of $M$ cancels out.
This bound is strictly decreasing in $\delta$. Thus, using the definition of $\phisharp(1/8)$ in \cref{eq:phisharp}, it suffices to find $\delta$ such that 
\begin{align*}
    \frac{c^2}{h^2\delta\sqrt{n}}\frac{2^{1-\rho}}{1-\rho} = \frac{1}{8}
\end{align*}
to reveal that 
\begin{align}
    \phisharp \lesssim \frac{c^2 2^{1-\rho}}{h^2\sqrt{n}(1-\rho)}.
    \label{eq:phisharp_value}
\end{align}
Using \cref{eq:s_t_solution} we get our result in \cref{eq:theorem}.
\section{Approximating the Identity}\label{app:ApproximateIdentity}
For a formal definition of the approximate identity introduced in \cref{sec:method}, first recall that the convolution of two Lebesgue integrable functions $g$ and $h$ in $L^1(\mathbb{R})$ is defined by
\begin{align*}
    (g*h)(y)\coloneq\int g(y-y')h(y')\dd y'.
\end{align*}
Now recall that a collection $\{K_h \colon h>0\}\subset L^1(\mathbb{R})$ of real-valued Lebesgue-integrable functions satisfying
\begin{itemize}
    \item[(i)] $\forall \; h>0, \int K_h(t) \dd t=1$,
    \item[(ii)] $\sup_{h>0} \int |K_h(t)| \dd t < \infty$,
    \item[(iii)] $\forall\varepsilon>0,\lim \sup_{h\to 0} \int K_h(t) \mathbbm{1}\{|t|>\varepsilon\} \dd t = 0$,  %
\end{itemize}
is called an approximate identity.
For instance, if $K_1$ is the standard Gaussian density and $K_h(\cdot)\coloneq h^{-1}K_1(\cdot/h)$, then (i) and (ii) are obviously met, as well as (iii) by the dominated convergence theorem.
The name `approximate identity' derives from the following property.

\begin{proposition}
    Let $\{K_h \colon h>0\}$ be an approximate identity.
    If $\varphi\in L^p(\mathbb{R})$ for some finite $p\geq1$, then $\lim\sup_{h\to 0}\int|(K_h*\varphi)-\varphi|^p \dd t=0$.
    Moreover, if $\varphi$ is bounded and uniformly continuous, then $\lim\sup_{h\to 0}\|(K_h*\varphi)-\varphi\|_\infty = 0$.
    \label{prop:approx_iden}
\end{proposition}
\noindent
Rather than recalling the full proof of \cref{prop:approx_iden}, let us give the simpler proof of the following result adapted to our case.
\begin{lemma}
    If $\fstar(\cdot|x)\in L^2([0,1])$ for (almost) every $x\in\X$ and if $\sup_{x\in\X}\|f(\cdot|x)\|_2<\infty$, then \mbox{$\lim\sup_{h\to0}\|\transformation(\cdot|\cdot)-\fstar(\cdot|\cdot)\|_2=0$}.
    \label{eq:transformation_h_limit}
\end{lemma}
\begin{proof}
    Set $h>0$ and observe that, for (almost all) $(x,y)\in\X\times[0,1]$,
    \begin{align*}
        \transformation(y|x) &= \E_{\Ph}\left(K_h(Y-Y')|Y'=y,X=x\right)\\
        &= \int_{[0,1]}K_h(\gamma-y)\fstar(\gamma|x) \dd \gamma\\
        &= \int K_h(t)\fstar(t+y|x)\dd t,
    \end{align*}
    with the convention that $\fstar(\gamma|x)=0$ for all $\gamma\notin[0,1]$.
    Therefore, using $\int K_h(t) \dd t=1$ (second equality), the Cauchy-Schwarz inequality (the inequality), and Fubini's theorem (third equality) yields
    \begin{align}
        \|\transformation(\cdot|\cdot)-\fstar(\cdot|\cdot)\|_2^2 &=
        \int_{\X}\int_{[0,1]}\left(\int K_h(t)\fstar(t+y|x)\dd t-\fstar(y|x)\right)^2 \fstar(x)\dd x \dd y\nonumber\\
        &=\int_{\X}\int_{[0,1]}\left(\int K_h(t)\left[\fstar(t+y|x)-\fstar(y|x)\right]\dd t\right)^2\fstar(x)\dd x\dd y\nonumber\\
        &\leq \int|K_h(t)|\dd t\times\int_{\X}\int_{[0,1]}\left(\int |K_h(t)|\left[\fstar(t+y|x)-\fstar(y|x)\right]^2\dd t\right)\fstar(x)\dd x\dd y\nonumber\\
        &=\int|K_h(t)|\dd t\times\int_{\X}\left(\int|K_h(t)|\times\|\tau_{t}\fstar(\cdot|x)-\fstar(\cdot|x)\|_2^2\dd t\right)\fstar(x)\dd x,
        \label{eq:norm_diff_transformed}
    \end{align}
    where $\tau_{t}\fstar\colon y\mapsto \fstar(t+y)$ and $\|\cdot\|_2$ denotes the $L^2(\mathbb{R})$-norm.

    Set arbitrarily $\varepsilon>0$ and $x\in\X$. Since $\fstar(\cdot|x)\in L^2([0,1])$, the triangle inequality yields
    \begin{align*}
        \|\tau_{t}\fstar(\cdot|x)-\fstar(\cdot|x)\|_2 \leq 2\|\fstar(\cdot|x)\|_2 \leq 2\sup_{x\in\X}\|\fstar(\cdot|x)\|_2<\infty
    \end{align*}
    for all $t\in\mathbb{R}$.
    Moreover, by \textcite[Lemma A.2]{Tsybakov2009}, there exists $t_x>0$ such that $\|\tau_{t}\fstar(\cdot|x)-\fstar(\cdot|x)\|_2^2\leq\varepsilon$ for all $|t|<t_x$. Then, by the definition of an approximate identity, there exists $h_x>0$ such that $0<h\leq h_x$ implies $\int|K_h(t)|\mathbbm{1}\{|t|\geq t_x\}\dd t\leq\varepsilon$.
    Consequently if $0<h\leq h_x$, then
    \begin{align*}
        \int|K_h(t)|\times\|\tau_t\fstar(\cdot|x)-\fstar(\cdot|x)\|_2^2\dd t
        &= \int|K_h(t)|\left(\mathbbm{1}\{|t|\geq t_x\} + \mathbbm{1}\{|t|<t_x\}\right)\times\|\tau_t\fstar(\cdot|x)-\fstar(\cdot|x)\|_2^2\dd t\\
        &\leq\left(4\|\fstar(\cdot|x)\|_2^2 + \sup_{h>0}\int|K_h(t)|\dd t\right)\varepsilon.
    \end{align*}
    Therefore, $x\mapsto \left(\int|K_h(t)|\times\|\tau_{t}\fstar(\cdot|x)-\fstar(\cdot|x)\|_2^2\dd t\right)\fstar(x)$ converges pointwise to zero as $h$ goes to zero. Because it is also upper-bounded by an integrable function independent of $h$ (a constant times $x\mapsto\fstar(x)$), the dominated convergence theorem guarantees that the RHS integral in \cref{eq:norm_diff_transformed} goes to zero as $h$ goes to zero, hence the result.
\end{proof}
\section{Hyperparameters}\label{app:hyper}
This section provides an overview over the implementation and hyperparameters of the \condensite methods as well as the other methods from the literature.
We specify any values set explicitly by us, for the remaining hyperparameters we keep the default values of the various implementations.
In the synthetic data (\cref{sec:synthetic}) and CPS (\cref{sec:IPUMS}) experiments, empirical ISE results are computed using a grid of $500$ points, for the satellite image data (\cref{sec:AGB}) we use only $30$ points to reduce the memory footprint. Note that any further data splitting within the methods described below are secondary splits of the 80\% training data resulting from the primary splits as described in the main text.

\subsection{\condensite}\label{app:hyper_condensite}

\paragraph{\condensiteNN.}
For \condensiteNN we use a fully connected
neural network with $5$ hidden linear layers, batch normalization \parencite{Ioffe2015}, and sigmoid-weighted linear unit activations \parencites[known as \textit{SiLU} or \textit{swish}, they were originally proposed in][]{hendrycks2016gelu}[and found to perform well empirically in][]{elfwing2018sigmoid,ramachandran2017searching}. 
Other common activation functions such as \textit{ReLU} \parencite{glorot2011deep} also work, but we think it is instructive to show how the predictor architecture can be used to shape the inductive bias of our method, in this case in favor of smoothness.
We use a batch size of $1024$ and $20$ neurons per hidden layer.
We train using the \textit{Adam} \parencite{kingma2015adam} optimizer with a learning rate of $10^{-3}$ and weight decay \parencite[see][]{loshchilov2019decoupled} of $10^{-4}$.
We use $80\%$ of the samples for training, and the remaining $20\%$ for cross-validation. For training, our objective is the mean squared error stated in \cref{eq:key_idea_empirical}. For cross-validation, we evaluate using the ISE as stated in \cref{eq:ISE}.
We train for at most $20$ epochs and stop the training if the best previous best cross-validation performance has not been exceeded for $5$ rounds.

\paragraph{\condensiteTree}
For \condensiteTree we use a \texttt{LightGBM}\footnote{https://pypi.org/project/lightgbm/4.6.0/} gradient boosted tree. 
As with \condensiteNN, we choose $h=0.01$ and $M=100$ and perform a $80\%$ to $20\%$ train and cross-validation split. We set the minimal number of data points per leaf to $50$, the number of bins to $40$ (the same as for \condensier and \lincde), and the number of leaves to $40$.

\paragraph{\condensiteCNN.}
This \condensite version consists of a CNN encoder and fully connected head, connected by skip connections. The CNN contains three convolutional layers with kernel size $3$, batch normalization, and max pooling with kernel size $2$.
The inputs are images with three channels, where the third channel consists in the value of the $Y'_{im}$ of the data point in question.
We then perform global average pooling on the activations of each convolutional layer, concatenate them with one another and again with the $Y'_{im}$ feature, and feed the result to the head.
The head itself consists of three hidden layers, the first with $128$ neurons, the subsequent ones with $64$.
For training, we use the same hyperparameters as for \condensiteNN, but reduce the batch size to $128$ for memory reasons.

\subsection{Other Methods}\label{app:hyper_literature}

\paragraph{\lincde.}
We mostly retain the default hyperparameters of the \lincde implementation\footnote{https://github.com/ZijunGao/LinCDE}, but change the number of trees from the default of $100$ to $200$, and the depth parameter for the individual trees from $1$ to $3$ so the expressivity of the estimator is not limited too severely. 

\paragraph{\drf.} 
We use the Python version of \drf with the default parameters of the implementation\footnote{https://github.com/lorismichel/drf} except for increasing the minimal node size parameter to $50$ to prevent overfitting, since we work with large datasets (the default value is $15$).
\drf returns an empirical probability mass function, so we apply Gaussian smoothing to obtain a density. We choose the bandwidth using Siverman's rule of thumb \parencite[see e.g.][]{silverman2018density}.

\paragraph{\flexcodeTree.}
This version of \flexcode uses an XGBoost\footnote{https://pypi.org/project/xgboost/3.0.2/} regressor  \parencite{Chen2016XGBoost} and a cosine basis system with 31 bases.

\paragraph{\flexcodeCNN.}
This version of \flexcode uses the same combination of CNN and fully connected head as \condensiteCNN as regressor for the coefficients of a cosine basis system with 31 bases, with the only difference being that here there are no $Y'_{im}$ features, and hence there is no third input channel. 
We use the same optimizer and training parameters as for \condensiteCNN.

\paragraph{\condensier.}
We choose the equal length binning method referred to as default in the documentation of the implementation\footnote{https://github.com/osofr/condensier}.
We increase the number of bins to $40$ (the default is $20$) to account for the complex nature of our datasets, since equal width bins put a hard limit on the expressivity of the method ($40$ bins is also the default value used by \lincde).

\paragraph{Feature extractor.}
In the image application presented in \cref{sec:AGB}, we use a feature extractor to transfer the two-channel $100\times 100$ image patches into $176$ features for \condensiteTree, \lincde, \drf, and \condensier.
The basis of the feature extractor is the same neural network also used for \flexcodeCNN.
Within each of the aforementioned methods, it is trained to predict AGB labels on the fraction of the training data used for fitting the model.
The features extracted for a given input image are the global average pooled activations of the convolutional layers which serve as input to the fully connected head.

\section{Proof of Concept and Hyperparameter Analysis (Additional Figures)}
This section provides additional figures and details complementing \cref{sec:synthetic} of the main text. 

\subsection{Synthetic Data -- Landmark Analysis}\label{app:landmarks}

This landmark analysis augments the ISE comparison presented in \cref{sec:PoC}.
We inspect the conditional densities fitted by the different methods at three different landmarks for each mechanism.
For the single relevant covariate setting we vary only the relevant covariate. For the data on manifold setting we vary the angle between the first two covariates. For the non-sparse data setting, we set all covariates to a specific value.
The landmarks are stated in \cref{tab:landmarks}. All unspecified covariates are fixed at zero.
\begin{table}[H]
    \centering
    \renewcommand{\arraystretch}{1.2}
    \begin{tabular}{l|lll}
        \textbf{Setting} & \textbf{Landmark 1} & \textbf{Landmark 2} & \textbf{Landmark 3}\\
        \hline
        Single relevant covariate & $x^{(1)}=0$ & $x^{(1)}=0.5$ & $x^{(1)}=1$\\
        \hline
        Data on manifold & $x^{(1)}=\cos\frac{2\pi}{6}$ & $x^{(1)}=\cos\pi$ & $x^{(1)}=\cos\frac{10\pi}{6}$\\
        & $x^{(2)}=\sin\frac{2\pi}{6}$ & $x^{(2)}=\sin\pi$ & $x^{(2)}=\sin\frac{10\pi}{6}$\\
        \hline
        Non-sparse data & $x=(-1,\dots,-1)$& $x=(0,\dots,0)$ & $x=(1, \dots, 1)$
    \end{tabular}
    \caption{Landmark values for each data-generating process to visualize the fitted estimators.}
    \label{tab:landmarks}
\end{table}
\noindent
\cref{fig:synth_lm_all} shows the conditional densities estimates for each method and landmark, and a histogram of $1000$ samples from the true density for comparison.
Overall, we observe good performance for both \condensite methods.
The data on manifold and non-sparse data settings prove somewhat challenging for \condensiteTree, whereas \condensiteNN gives a close fit everywhere.
\lincde performs well in most settings but struggles somewhat with the magnitude of the shifts along the horizontal axis for the non-sparse data.
In this respect it is similar to \condensiteTree, although it provides a smoother fit.
Similarly to the \condensite methods, \drf provides an excellent visual fit on all single relevant covariate landmarks, but underestimates the magnitude of the shifts in some of the other settings.
\flexcode fails to adapt to the heteroskedasticity of the single relevant covariate landmarks, and struggles in the data on manifold setting.
Its estimates also exhibit noticeable spurious peaks in multiple instances.
A larger number of bases might give it greater flexibility to adapt to concentrated densities, but would likely also incur even more severe spurious peaks.
Lastly, \condensier consistently estimates the location of the mode well, although it under-smooths rather severely in the non-sparse data setting.
In summary, the examples confirm that both versions of \condensite, as well as the other methods from the literature, capture the essential tendencies of the conditional distributions in our synthetic data examples correctly.
\begin{figure}[H]
    \centering
    \includegraphics[width=.9\linewidth]{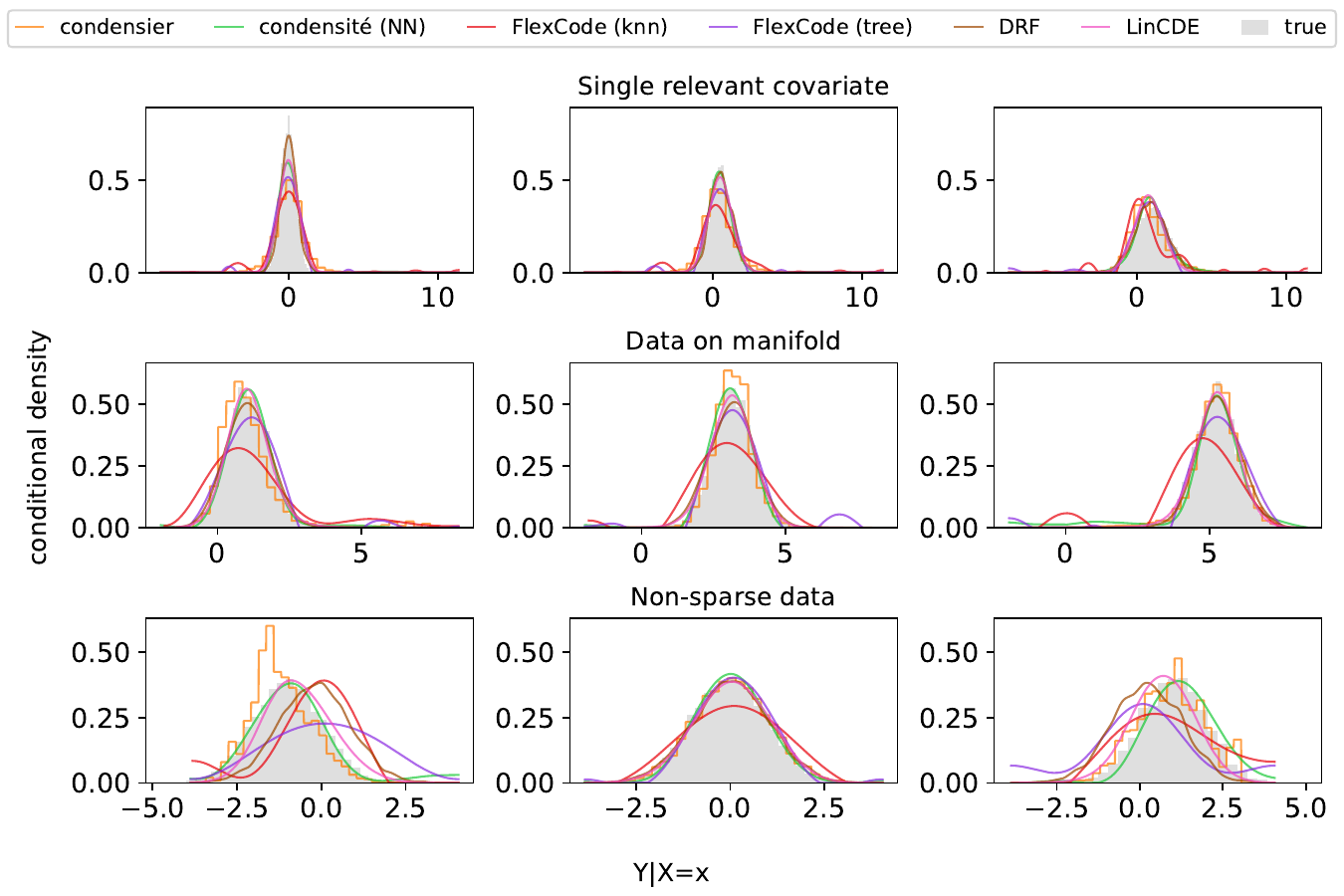}
    \caption{Comparison of the estimates to a sample from the true conditional density on a grid of points.}
    \label{fig:synth_lm_all}
\end{figure}

\subsection{The Effect of \texorpdfstring{$M$}{\textit{M}} and Its Interaction With \texorpdfstring{$h$}{\textit{h}} -- {\condensiteTree}}\label{app:h_M_interplay}
\cref{fig:h_M_interplay_Tree} below shows the experiment exploring the interplay of $M$ and $h$ for \condensiteTree. It complements \cref{fig:h_M_interplay_NN} in \cref{sec:h_M_interplay} of the main text.
\begin{figure}[H]
    \centering
    \begin{subfigure}{.32\linewidth}
        \centering
        \includegraphics[width=\linewidth]{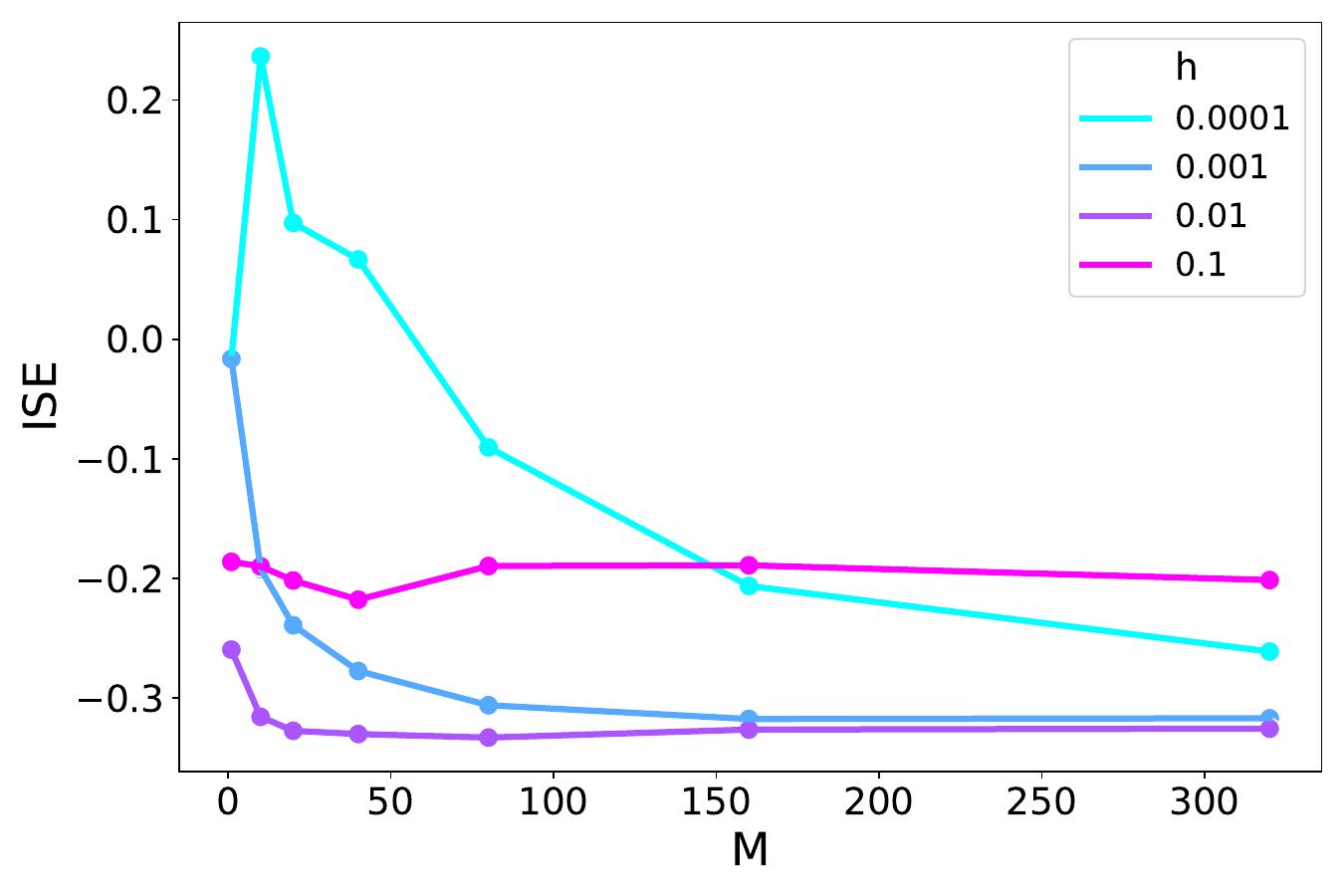}
        \caption{Single relevant covariate.}
    \end{subfigure}
    \begin{subfigure}{.32\linewidth}
        \centering
        \includegraphics[width=\linewidth]{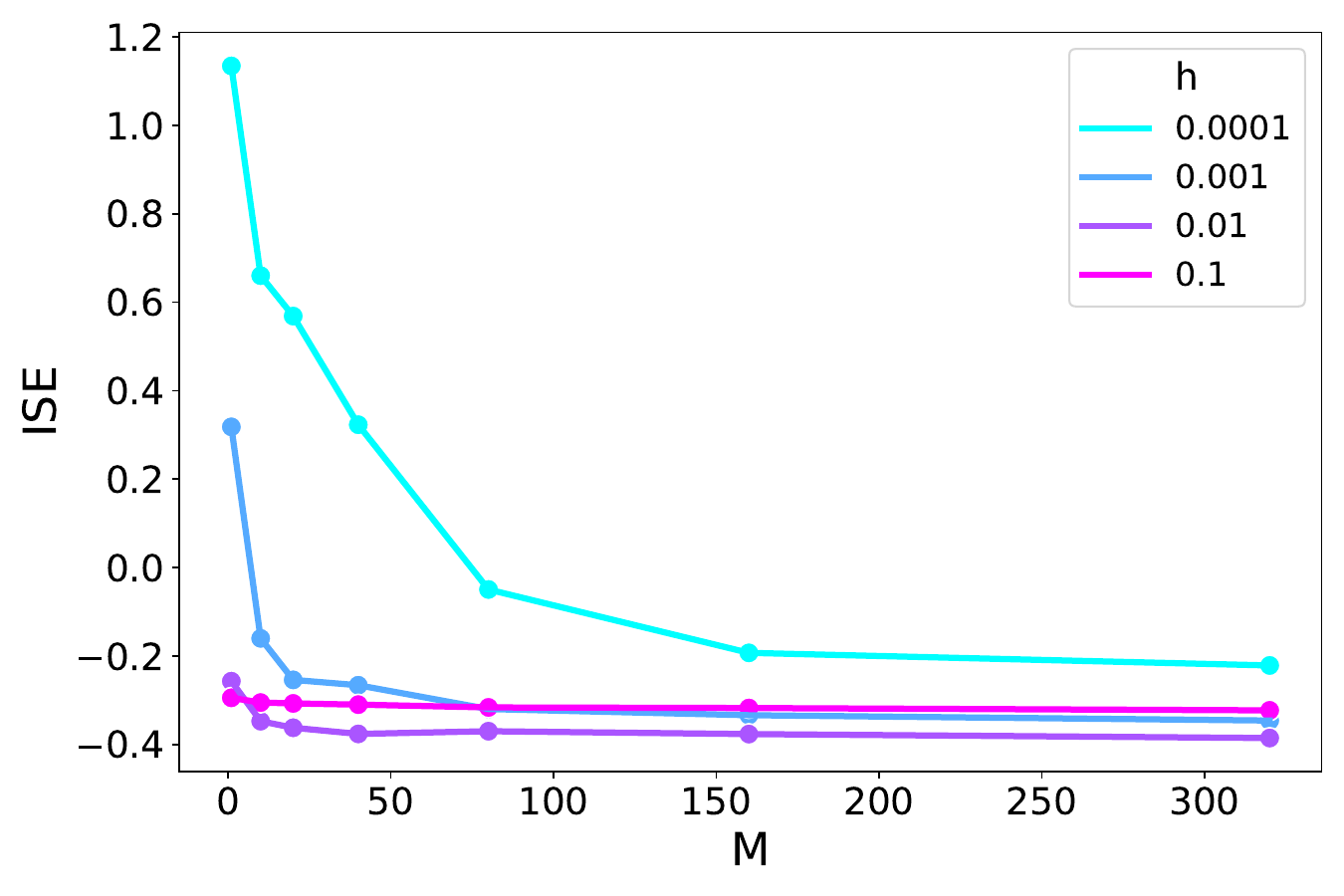}
        \caption{Data on manifold.}
    \end{subfigure}
    \begin{subfigure}{.32\linewidth}
        \centering
        \includegraphics[width=\linewidth]{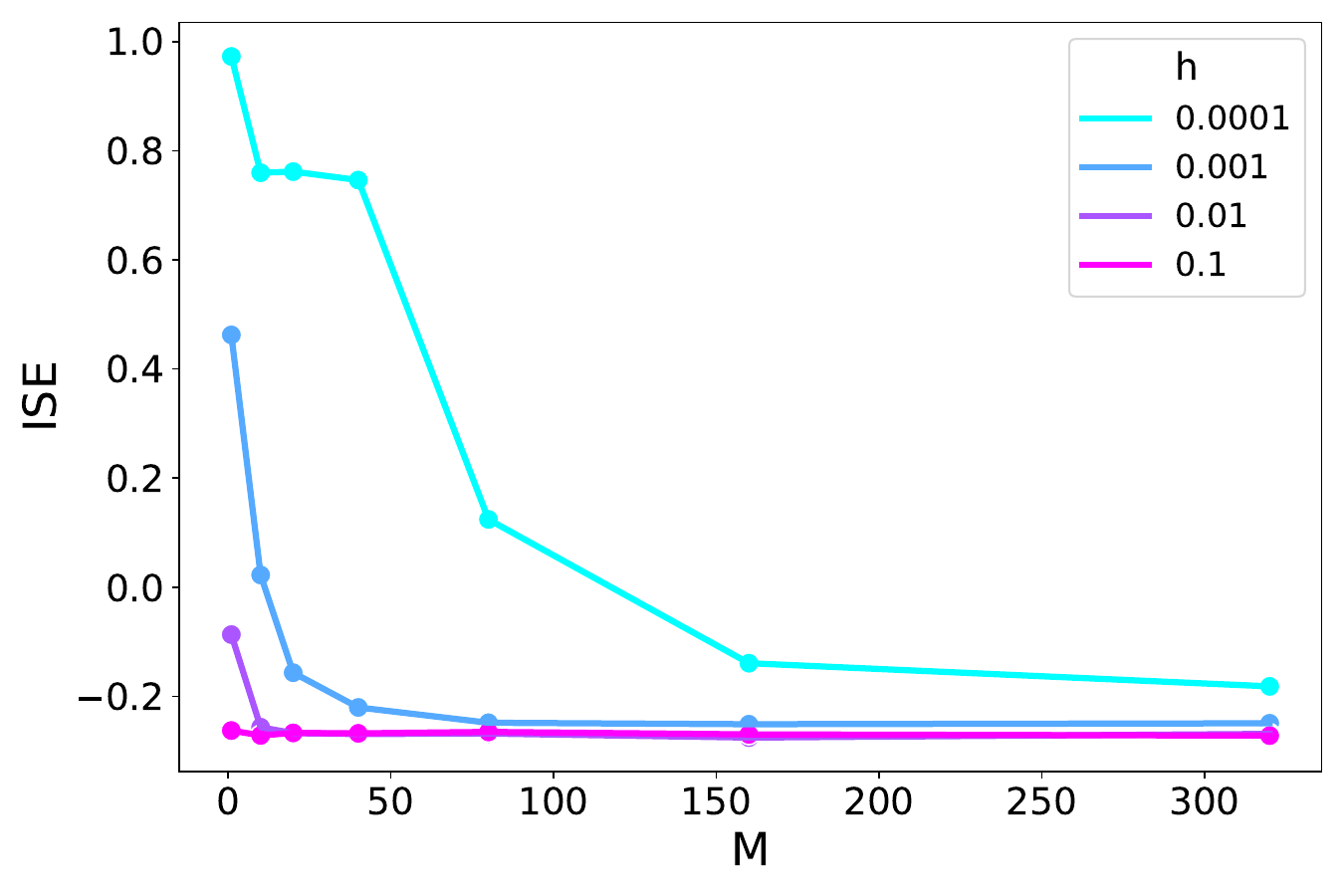}
        \caption{Non-sparse data.}
    \end{subfigure}
    \caption{\condensiteTree: ISE for different $M$ and $h$ on the synthetic data-generating mechanisms.}
    \label{fig:h_M_interplay_Tree}
\end{figure}

\section{The IPUMS-CPS Dataset (Preprocessing and Additional Figures)}\label{app:IPUMS}

This section contains details on our compilation of the CPS dataset, as well as a landmark analysis for the other methods from the literature.

\subsection{Dataset and Preprocessing}\label{app:ipums_pp}

For the evaluation of our methods on real-world data in \cref{sec:real_world} of the main text, we use CPS records from 2024 for which the Annual Social and Economic Supplements (ASEC) are available, meaning from March 2024.
We create a dataset based on the variable selection outlined in \cref{tab:ipums_vars}.
\begin{table}[H]
    \centering
    \begin{tabular}{ll}
        CBSASZ & (population size of household location area)\\
        AGE & (age at last birthday) \\
        SEX & (sex)\\
        NCHILD & (number of own children in household)\\
        YRIMMIG & (year of immigration into the US)\\
        NATIVITY & (foreign birthplace or parentage)\\
        LABFORCE & (labor force status)\\
        UHRSWORKT & (hours usually worked per week at all jobs)\\
        RACE & (race)\\
        EMPSTAT &  (employment status) \\
        CLASSWKR & (class of worker) \\
        EDUC & (educational attainment) \\
        INCTOT & (total personal income)
    \end{tabular}
    \caption{IPUMS-CPS variable selection.}
    \label{tab:ipums_vars}
\end{table}
\noindent
We perform the following preprocessing steps:
\begin{enumerate}
    \item keep only observations that are part of the ASEC,
    \item drop all observations with
    \begin{itemize}
        \item INCTOT negative, unknown, or in excess of $300000\$$,
        \item unknown or below secondary school education,
        \item unknown LABFORCE or NATIVITY,
    \end{itemize}
    \item re-code SEX and LABFORCE to $\{0,1\}$ dummy variables,
    \item one-hot encode RACE and EMPSTAT by the first digit,
    \item one-hot encode CLASSWKR,
    \item map EDUC to years of education,
    \item create separate dummy for UHRSWORKT code $997$ (varying hours), and set values $997$ and $999$ (not in universe) to $0$.
\end{enumerate}
This leaves us with a total of $113104$ observations of $26$ covariates (not including INCTOT).
The variables AGE, NCHILD, YRIMMIG, NATIVITY, UHRSWORKT, and EDUC\_YEARS (derived from EDUC) are multi-valued, with the remaining ones being binary.
All variables are encoded numerically.
Importantly, whenever we draw a sample for learning or evaluation, we resample with replacement weighted by ASECWT (the ASEC weighting factor) to obtain a representative sample.

\subsection{Landmark Analysis for Realistic Use-Cases (Other Methods)}\label{app:ipums_landmarks}

\cref{fig:skill_prem_lit} and \cref{fig:lm_metro_lit} below show the landmark analyses for the other methods from the literature corresponding respectively to \cref{fig:skill_prem} and \cref{fig:lm_metro} in the main text.
\drf and \condensier fit areas of high density well, but seemingly at the cost of overfitting other areas, resulting in high-variance estimates.
\lincde provides a smooth fit but does not manage to fit highly concentrated densities.
Despite performing worse in ISE than \condensiteCNN, \condensiteTree, and \flexcodeTree, the other methods from the literature exhibit the same general trends in line with the manually constructed local empirical density.
\vskip-.75em
\begin{figure}[H]
    \centering
    \includegraphics[width=\linewidth]{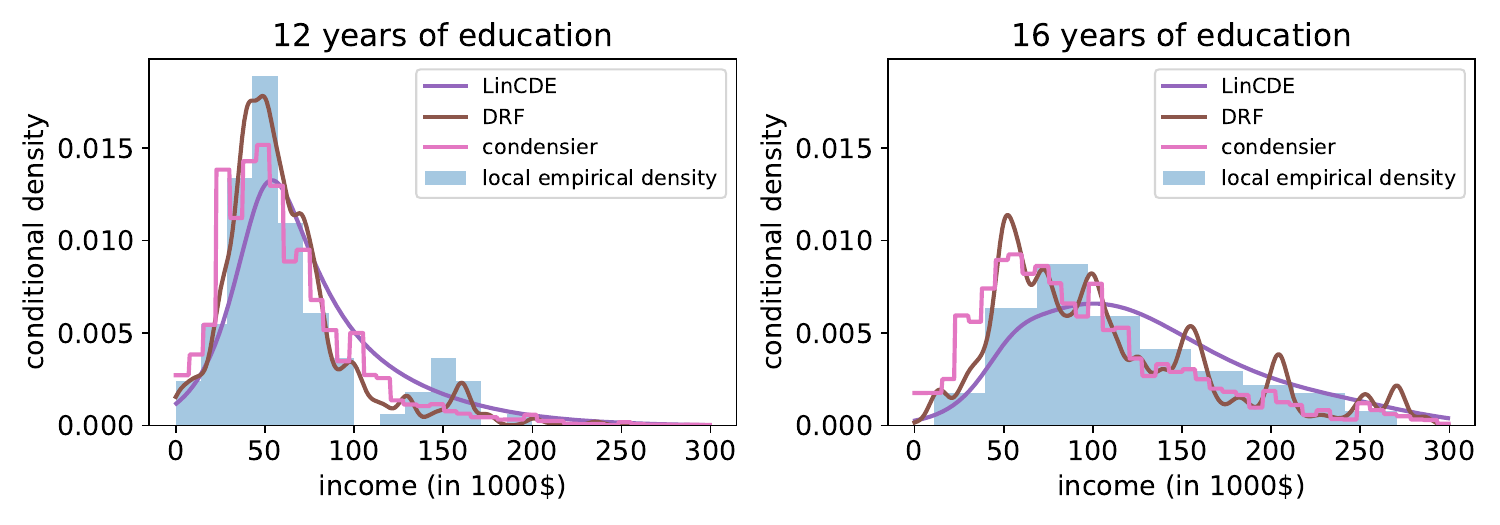}
    \vskip-1em
    \caption{Conditional income density for 12 and 16 years of education with otherwise identical covariates.}
    \label{fig:skill_prem_lit}
\end{figure}
\vskip-1.25em
\begin{figure}[H]
    \centering
    \includegraphics[width=\linewidth]{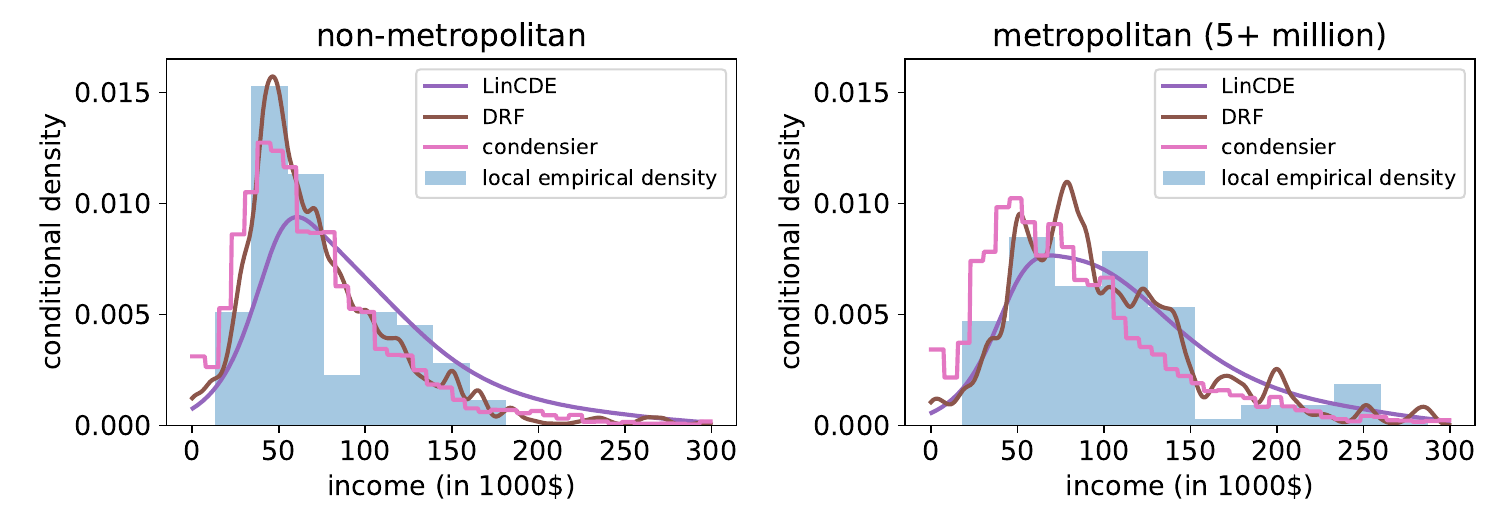}
    \vskip-1em
    \caption{Conditional income density for metropolitan and non-metropolitan inhabitants with otherwise identical covariates.}
    \label{fig:lm_metro_lit}
\end{figure}

\section{ESA ICC Satellite Imaging Data (Preprocessing and Additional Figures)}

\subsection{Dataset and Preprocessing}\label{app:ESA_pp}
The following are details on the dataset and additional results for the AGB satellite image application presented in \cref{sec:AGB}.

\paragraph{Label data.}
AGB is a measure of biomass density, defined as dry weight of live wood per unit area.
We derive our labels from the biomass dataset provided by the ESA Biomass Climate Change Initiative \parencite{esacci2025AGB}.
This dataset provides labels computed from satellite images of different wavelengths by a theoretically motivated and empirically calibrated algorithm. 
Each label corresponds to a specific geographic area of $100\times100$ meters.
The AGB measurements provided are given in megagrams (tons) per hectare and constitute yearly averages.
Our experiment sets out to predict biomass measurements from satellite images of the type used for the creation of the \cite{esacci2025AGB} dataset, essentially expanding upon the role of the theoretical model used in the original study,
which provides mean estimates, not conditional densities.
To this end we choose the two rectangular coordinate regions given in \cref{tab:coordinate_rectangle}, one for training, and the other for testing out-of-sample.
\begin{table}[H]
    \centering
    \begin{tabular}{lcccc}
         & north-west & north-east & south-east & south-west\\
        \textbf{training region} & $[132.5, -14.65]$ & $[133.75, -14.65]$ & $[133.75, -15.65]$ & $[132.5, -15.65]$\\
        \textbf{test region} & $[132.75,  -13.25]$ & $[134.0, -13.25]$ & $[134.0, -14.25]$ & $[132.75,  -14.25]$ \\
    \end{tabular}
    \caption{Coordinate rectangle.}
    \label{tab:coordinate_rectangle}
\end{table}
\noindent
Both of these are in the Australian Northern Territory. The training patch covers an area to the south-east of Kakadu national park, and the test patch an area to the south-east of the town Katherine.
The location is chosen for its low biomass, which can be captured by C-band imaging, and little seasonal change in AGB.

\paragraph{Feature data.}
To obtain features for each label patch, we retrieve a ground range detected Sentinel-1A satellite image taken in interferometric wide swath mode from the GEODES portal \parencite{geodes2025}.
We choose the images in \cref{tab:sat_images}, which contain the corresponding patches described in \cref{tab:coordinate_rectangle}.
They have been taken on August 15th 2020, that is during the dry season where foliage is low and C-band measurements are most informative. We ensure cloud cover is less than 5\%.
We work with an ortho-rectified version of the images as can be obtained through the GEODES portal, and use the \textit{VV} and \textit{VH} bands.
The resolution of the images is $10\times 10$ meters, their IDs are given in \cref{tab:sat_images}.
\begin{table}[H]
    \centering
    \begin{tabular}{ll}
        \textbf{training} & S1A\_IW\_GRDH\_1SDV\_20200815T204041\_20200815T204106\_033922\_03EF62\_FF6E\\
        \textbf{test} & S1A\_IW\_GRDH\_1SDV\_20200815T204106\_20200815T204131\_033922\_03EF62\_BB0B
    \end{tabular}
    \caption{IDs of the Sentinel-1A satellite images used to construct our features.}
    \label{tab:sat_images}
\end{table}
The images are shown in \cref{fig:training_feat_img} and \cref{fig:test_feat_img} for illustration; note that our preprocessing is minimal and not aimed at visualization.
\begin{figure}[H]
    \centering
    \begin{subfigure}{.49\linewidth}
        \centering
        \includegraphics[width=.8\linewidth]{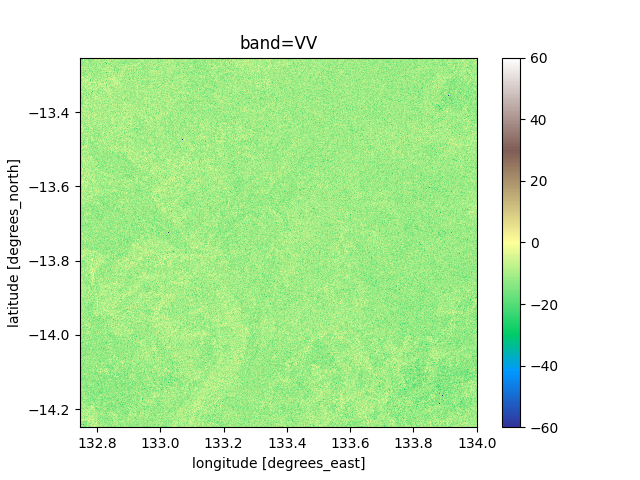}
    \end{subfigure}
        \begin{subfigure}{.49\linewidth}
        \centering
        \includegraphics[width=.8\linewidth]{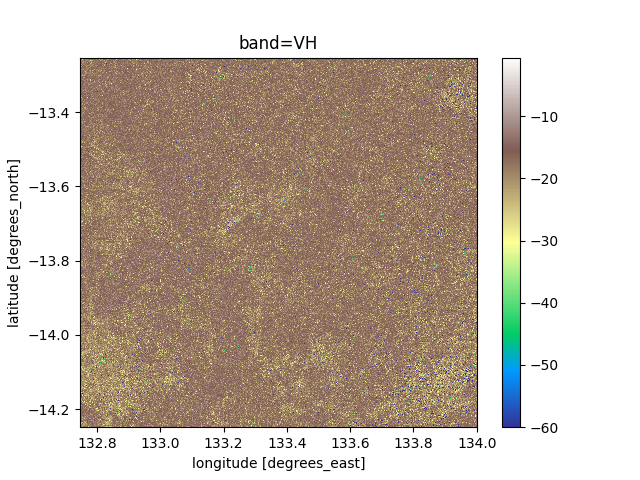}
    \end{subfigure}
    \caption{\textbf{Training region.}  Images (in decibel) used as the basis for our feature patches.}
    \label{fig:training_feat_img}
\end{figure}
\begin{figure}[H]
    \begin{subfigure}{.49\linewidth}
        \centering
        \includegraphics[width=.8\linewidth]{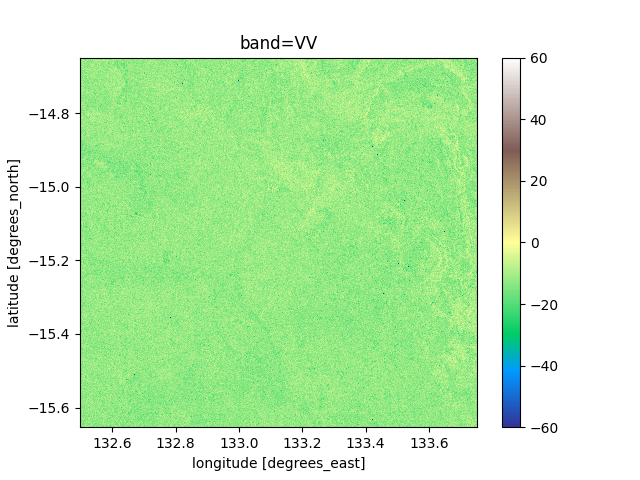}
    \end{subfigure}
        \begin{subfigure}{.49\linewidth}
        \centering
        \includegraphics[width=.8\linewidth]{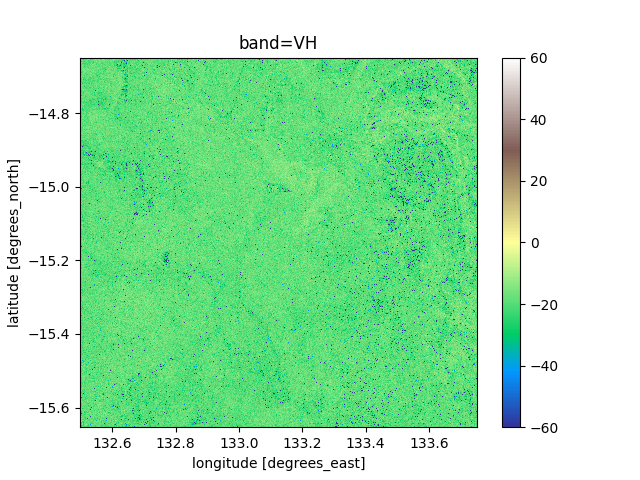}
    \end{subfigure}
    \caption{\textbf{Test region.} Images (in decibel) used as the basis for feature patches.}
    \label{fig:test_feat_img}
\end{figure}

\paragraph{Compiling our dataset.}
To compile our dataset, we perform the following steps:
\begin{enumerate}
    \item fuse AGB labels into square kilometer patches by taking the average,
    \item log-transform target AGB values by $x\mapsto \log(1+x)$ to avoid dominance of near-zero values,
    \item log-transform Sentinel-1A image pixels by $x\mapsto 10\log_{10}(x)$ to obtain values in decibel,
    \item identify the corresponding $100\times100$ pixel satellite image patch for each fused AGB label by reprojecting into EPSG:3577 (drop all patches with missing data).
\end{enumerate}
\noindent
After the first step, each AGB label corresponds to a one square kilometer patch, which makes for a $100\times 100$ pixel patch in the satellite images. 
This results in $14653$ observations for the training region, and $14683$ observations for the test region.
Each observation is a $100\times100$ image with two channels, \textit{VV} and \textit{VH}.

\newpage
\subsection{Visualization and Qualitative Assessment (Training Region)}\label{app:viz_AGB}
\cref{fig:panel_train} below contains the training region counterparts to the test region plots in \cref{fig:panel_test}.
\begin{figure}[H]
    \centering
    \begin{subfigure}{\linewidth}
        \centering
        \includegraphics[width=\linewidth,trim={0 1.4cm 0 .95cm},clip]{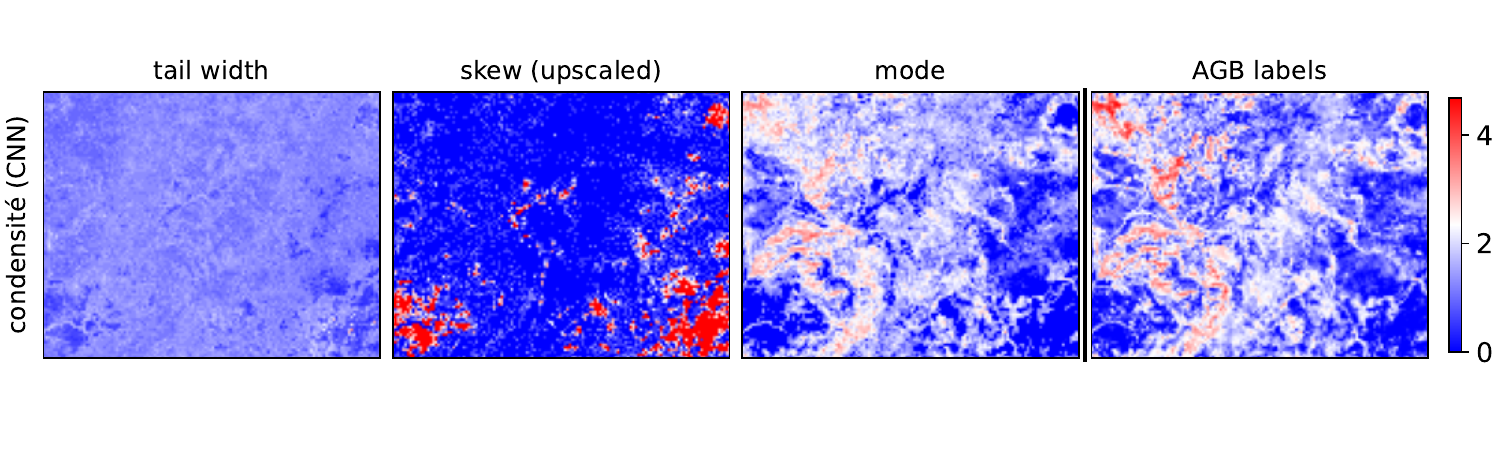}
    \end{subfigure}
        \begin{subfigure}{\linewidth}
        \centering
        \includegraphics[width=\linewidth,trim={0 1.5cm 0 1.5cm},clip]{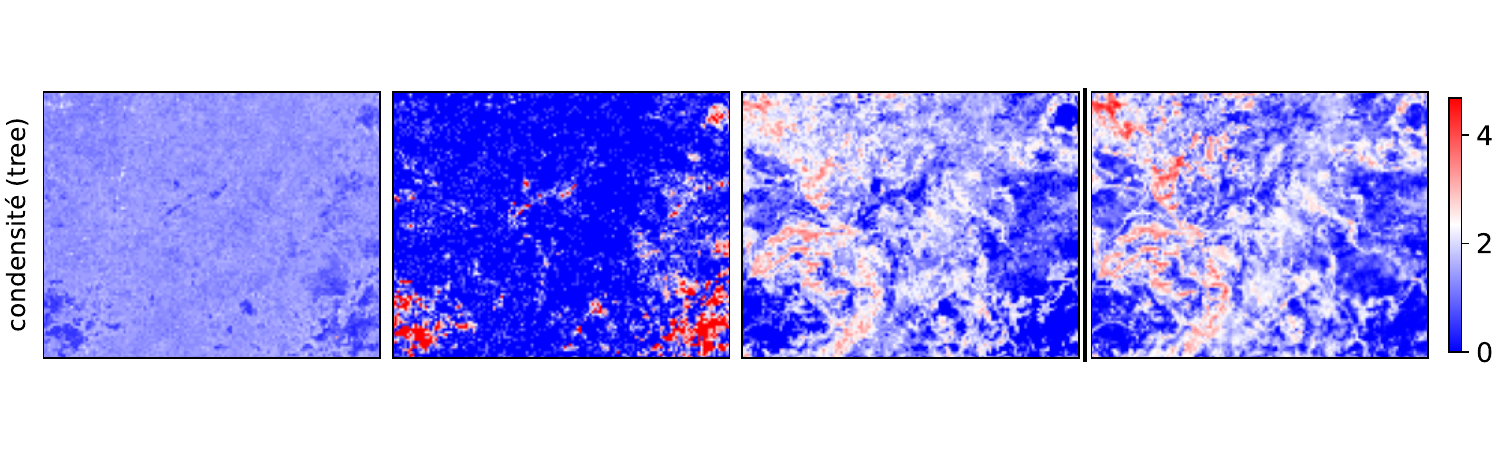}
    \end{subfigure}
        \begin{subfigure}{\linewidth}
        \centering
        \includegraphics[width=\linewidth,trim={0 1.5cm 0 1.5cm},clip]{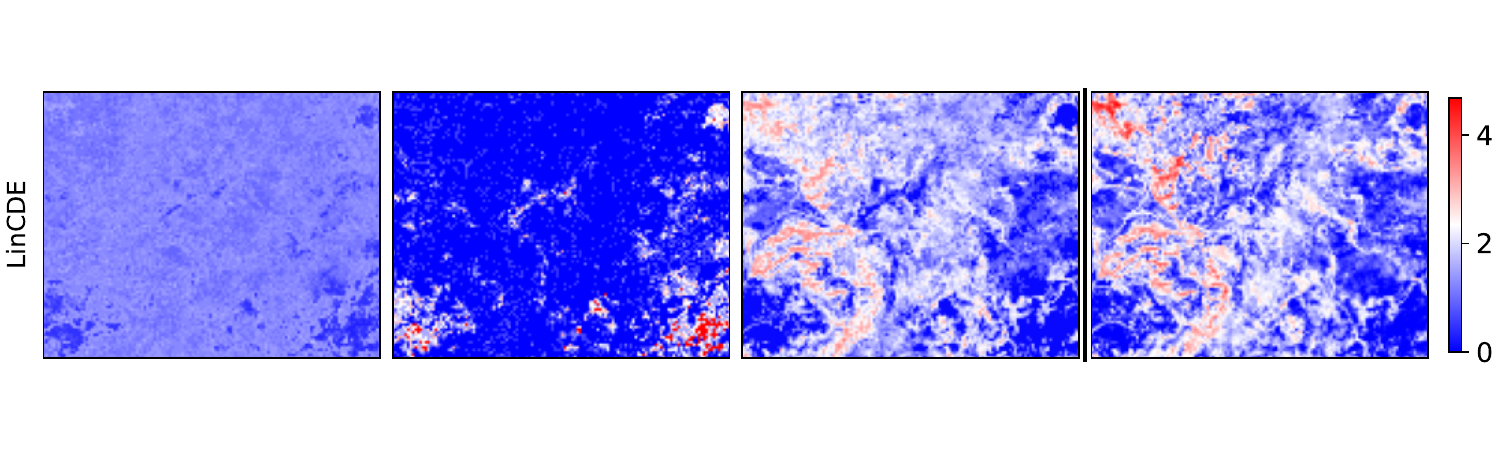}
    \end{subfigure}
    \begin{subfigure}{\linewidth}
        \centering
        \includegraphics[width=\linewidth,trim={0 1.5cm 0 1.5cm},clip]{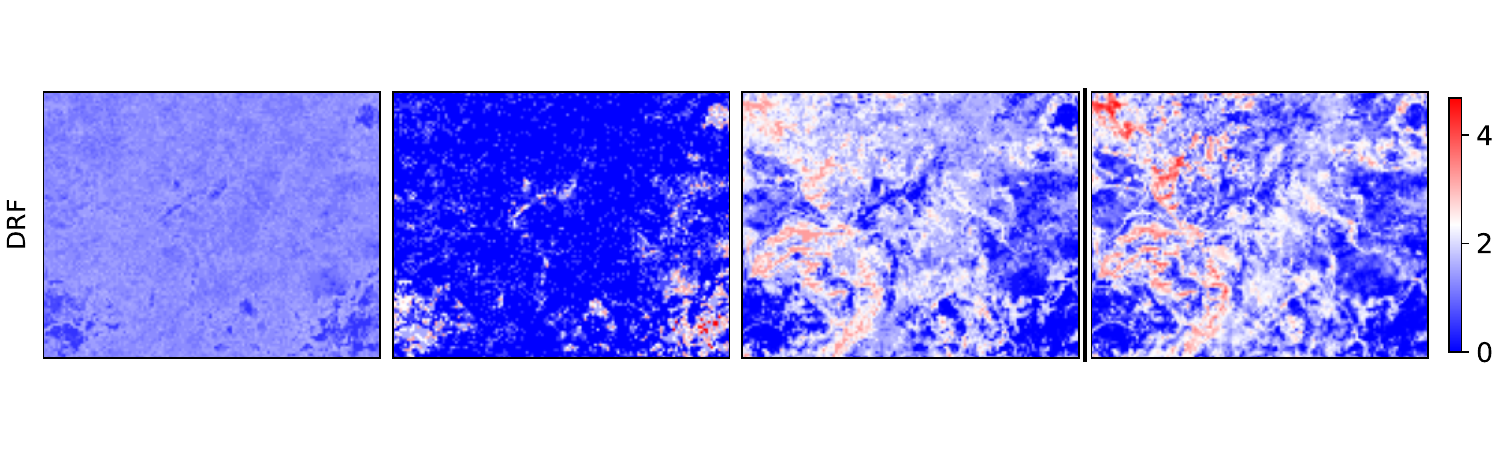}
    \end{subfigure}
        \begin{subfigure}{\linewidth}
        \centering
        \includegraphics[width=\linewidth,trim={0 1.5cm 0 1.5cm},clip]{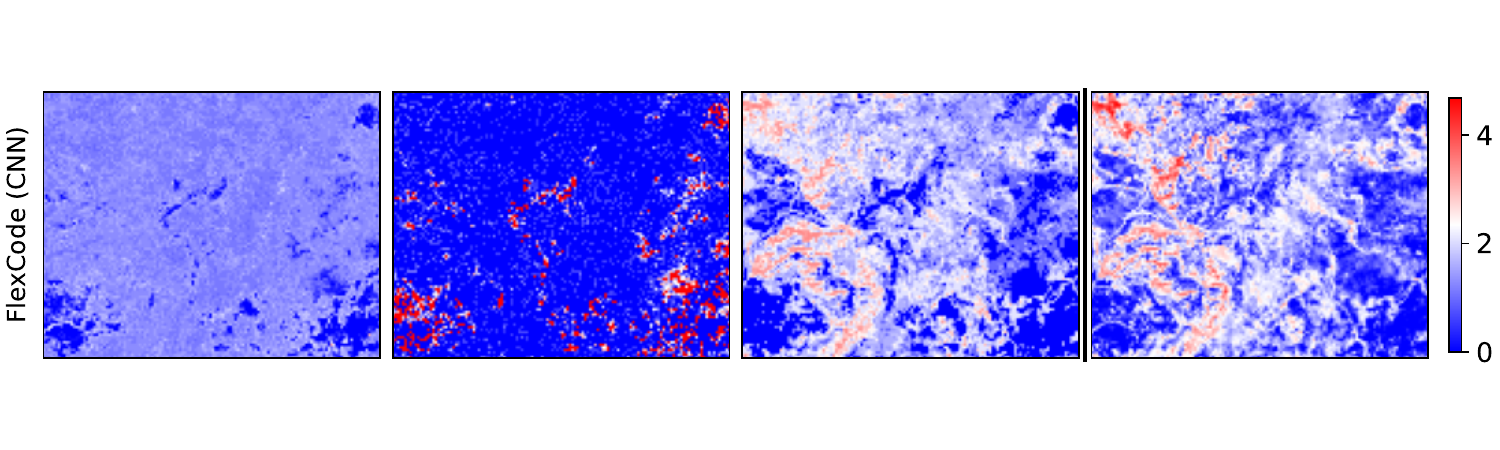}
    \end{subfigure}
    \begin{subfigure}{\linewidth}
        \centering
        \includegraphics[width=\linewidth,trim={0 1.5cm 0 1.5cm},clip]{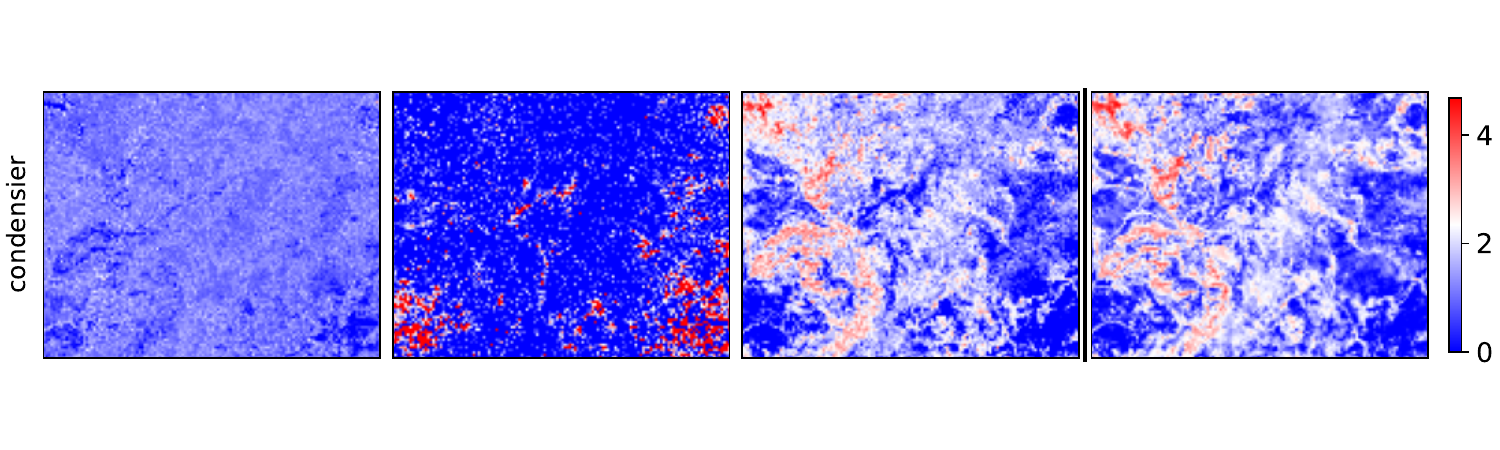}
    \end{subfigure}
    \caption{\textbf{Training region.} Visualization of method estimate summary statistics and labels.}
    \label{fig:panel_train}
\end{figure}

\end{document}